\documentclass[accepted]{uai2026} 
                        
\usepackage[american]{babel}

\usepackage{natbib} 
\usepackage{booktabs} 
\usepackage{tikz} 
\usepackage{enumitem}
\usepackage{graphicx,wrapfig}
\usepackage[caption = false]{subfig}
\usepackage{subcaption}
\usepackage{dcolumn}
\usepackage[utf8]{inputenc}
\usepackage{blindtext}
\usepackage{xcolor}
\usepackage{algorithm} 
\usepackage{algpseudocode} 
\usetikzlibrary{trees}
 \usepackage{soul}
 \usepackage{bm}
 \usepackage{mathtools,amsmath,amssymb,amsthm}
 \usepackage{float}
 \usepackage[caption = false]{subfig}
 \usepackage{graphicx}
\usepackage{multirow}
\usetikzlibrary{patterns,shapes.geometric, arrows.meta,positioning}
\newtheorem{theorem}{Theorem}

\newtheorem{corollary}{Corollary}[theorem]
\theoremstyle{definition}
\newtheorem{definition}{Definition}
\newtheorem{lemma}{Lemma}
\tikzset{
	circle node/.style = {circle, draw, inner sep=1pt, minimum size=6pt},
	square node/.style = {rectangle, draw, inner sep=1pt, minimum size=6pt}
}

\title{A resource rational account of the explore-exploit dilemma via metareasoning}

\author[1]{\href{mailto:<pg2991@nyu.edu>?Subject=Your UAI 2026 paper}{Prakhar Godara}{}}
\author[2]{Tilman Diego Alem\'an}
\affil[1]{%
    Department of Psychology\\
  New York University\\
  New York, 10003 
}
\affil[2]{%
    Institut f\"ur Geometrie und Praktische Mathematik\\
  RWTH-Aachen University\\
   Aachen, 52056
}
  
\begin{document}
\maketitle

\begin{abstract}
  \textit{Reasoning} may be viewed as an algorithm $P$ that makes a choice of an action $a^* \in \mathcal{A}$, aiming to optimize some outcome. However, executing $P$ itself bears costs (time, energy, limited capacity, etc.) and needs to be considered alongside explicit utility obtained by making the choice in the underlying decision problem. Finding the right $P$ can itself be framed as an optimization problem over the space of reasoning processes $P$, generally referred to as \textit{metareasoning}. Conventionally, human metareasoning models assume that the agent knows the transition and reward distributions of the underlying MDP. This paper generalizes such models by proposing a meta Bayes-Adaptive MDP (meta-BAMDP) framework to handle metareasoning in environments with unknown reward/transition distributions, which encompasses a far larger and more realistic set of planning problems that humans and AI systems face. As a first step, we apply the framework to Bernoulli bandit tasks. Owing to the meta problem's complexity, our solutions are necessarily approximate. However, we introduce two novel theorems that significantly enhance the tractability of the problem, enabling stronger approximations that are robust within a range of assumptions grounded in realistic human decision-making scenarios. These results offer a resource-rational perspective and a normative framework for understanding how humans trade-off exploration and exploitation under cognitive constraints, as well as providing experimentally testable predictions about human behavior in Bernoulli Bandit tasks. 
\end{abstract}

\section{Introduction}

 In decision-making scenarios, \textit{reasoning} may be viewed as an agent executing an algorithm $P$ that selects an action $a^* \in \mathcal{A}$ that optimizes some outcome, for instance maximizing the value function of a Markov decision process (MDP). Similarly, \textit{metareasoning} \citep{RUSSELL1991,hay2014selecting} can be construed as an algorithm $P^\text{meta}$ such that it  selects a reasoning algorithm $P^*$, aiming to optimize some performance measure. The performance measure includes both the expected reward $P^*$ obtains in the underlying decision problem, as well as the costs (time, energy, etc.) of executing $P^*$.

This description of metareasoning is sufficiently broad to encompass  several domains, such as meta-optimization, hyperparameter optimization \citep{mercer1978,Smit2009,huang2019}, etc. Recently, metareasoning has also been studied in the context of human behavior \citep{lieder2018rational,callaway2022rational,lieder2014algorithm}. The motivation behind studying {\it normative} metareasoning in humans is as follows: just as human behavioral choices are (arguably) subjected to selection pressures and therefore close to optimal in a wide variety of tasks (and hence intelligent), the reasoning process humans use to arrive at good behavioral choices is itself under selection pressure and thus also close to optimal. Crucially, discussions on metareasoning typically focus on characterizing the properties of the solution to the meta-optimization problem, while neglecting the implementational details of the actual optimization procedure, whether done offline through evolutionary or developmental processes, or done online by the agent itself. Our work continues this philosophy, and builds on prior work by significantly widening the space of problems amenable to  metareasoning modeling approach.

\subsection{Related work and contributions}

Recent efforts in studying metareasoning in humans have focused on planning problems, i.e. the underlying, externally-defined problem is viewed as an MDP \citep{callaway2022rational,lieder2018rational} whereby the agent knows the true transition/reward distributions but not the  optimal value function. The agent therefore engages in some reasoning to improve upon its current policy/value function, for instance via a(n asynchronous) policy/value iteration algorithm. Metareasoning, in such a scenario, then concerns itself with finding the states (of the MDP) on which to perform the policy/value iteration \citep{D2021}. As it turns out, the metareasoning problem can be framed as an MDP itself, albeit with a significantly larger action space.

In order to extend metareasoning to a more general set of decision-making problems, the assumption of known transition dynamics needs to be dropped, i.e. the transition/reward distribution will have to be learned online. For this reason, instead of considering the underlying problem as an MDP, we consider it to be a Bayes-adaptive Markov decision process (BAMDP). Importantly, the theoretical benefit of using a BAMDP over an MDP (with evolving transition dynamics as in the case of Dyna-Q \citep{sutton1991dyna}), is that the former incorporates the evolution of the belief (about both the state transition dynamics) within its transition function. This allows the transition dynamics to be known, despite the actual dynamics of the environment not being known. This key property of BAMDPs allows us to conveniently formulate the metalevel decision problem in a manner very similar to conventional meta-MDP formulations \citep{callaway2022rational,hay2014selecting,lin2015}, albeit with an additional belief space (see \ref{app:relation_to_MMDP} for further comparisons with meta-MDPs).

However, BAMDPs have been known for being prohibitively hard to solve in practice, owing to their  large (infinite dimensional) state space \citep{duff2002}. These concerns are further exacerbated by the fact that we are interested in a meta version of the BAMDP, which (in the general case) is known to be always harder\footnote{As the metareasoning problem compounds the complexity of the base task, the meta version is always at least as hard — and often harder — than the original problem (see Sec. 3).} than the underlying problem \citep{RUSSELL1991}.  This suggests one to look for good approximations to solve the meta problem, as has been pointed out before \citep{hay2014selecting,lin2015}.

Indeed, most recent studies have largely assumed greedy or approximate solutions for the metalevel problem \citep{lieder2017strategy}, and/or have been restricted to small state spaces \citep{callaway2022rational,jain2023computational} due to computational limitations. In this work, we make significant theoretical advances by developing methods that make the metalevel problem tractable for substantially larger state spaces.

Finally, we study the effects of metareasoning on agents performing a Bandit task, as they provide the simplest setting involving an explore-exploit dilemma. For this reason, they also receive significant attention in experimental and theoretical cognitive science studies \citep{zhang2013,STEYVERS2009}. Typically, human behavior in these tasks is compared to heuristic behavioral policies. We will show that our approach not only aligns with observed qualitative features of human behavior under cognitive constraints as reported in recent studies \citep{BROWN2022,wu2022,wilson2014,cogliati2019should,otto2014physiological}, but also offers a normative explanation for observed human adaptation to cognitive load.  We also provide novel quantitatively testable predictions for human behavior in Bernoulli Bandit tasks. To our knowledge, this is the first study to provide a resource-rational account of human exploration. 

To summarize, we (1) generalize metareasoning to unknown dynamics via a novel meta-BAMDP framework; (2) prove two new theorems that make the meta-problem (exponentially more) tractable; (3) derive a tractable algorithm based on the said theorems to solve the meta-BAMDP; (4) recover known human exploration effects under resource constraints; and (5) make new, testable predictions for human behavior in Bernoulli bandit tasks.

\section{Meta-Bayes-Adaptive Markov Decision Process - meta-BAMDP} \label{sec:m-BAMDP-defn}

We now present our definition of a meta-BAMDP (its relation to meta-MDP formulations is further discussed in \ref{app:relation_to_MMDP}). In order to draw a contrast with the definitions of (finite horizon) MDPs and BAMDPs, we request the reader to refer to \ref{app:MDP_def} and \ref{app:BAMDP_def} respectively.

A meta-BAMDP is defined as a tuple $(\mathcal{S}_{MB},\mathcal{A}_{MB},\mathcal{P}_{MB},\mathcal{R}_{MB},\mathcal{K})$, where:

\begin{itemize}
    \item The state space is given by $\mathcal{S}_{MB} = \mathcal{S}\times \mathcal{B} \times \tilde{\mathcal{B}}$. Here, $\mathcal{S}$ is the state of the environment, $\mathcal{B}$ is the set of beliefs $b\in \mathcal{B}$ representing the agent's state of knowledge regarding the transition and reward distributions. $\tilde{\mathcal{B}}$ is the space of planning-beliefs. Abstractly, planning-belief $\tilde b \in \tilde{\mathcal{B}}$ represents the intermediate computational states of a planning algorithm, thereby representing the \textit{extent} to which the agent has engaged in planning. Eg - if the planning algorithm finds the optimal path on a DAG $G$, then its intermediate computational states can be viewed as sub-DAG's of $G$. A concrete example follows in the next section.

    \item The agent's action space $\mathcal{A}_{MB} = \mathcal{A} \cup \mathcal{C}$, where $\mathcal{A}$ corresponds to the physical actions and $\mathcal{C}$ represents the computational actions. Both these actions differ in the kinds of transitions they cause. A physical action causes transitions in the state $s\in \mathcal{S}_{MB}$. A computational action causes transitions only in the planning-beliefs $\tilde b \in \tilde{\mathcal{B}}$.

    \item The meta-BAMDP transition function $\mathcal{P}_{MB}$ is  given by 
\begin{equation}\label{eq:MB_trans_func}
   \mathcal{P}_{MB} =\begin{cases} 
      \mathcal{P}_A(s',b',\tilde{b}'|s,b,\tilde{b},a), & a\in \mathcal{A}, \\
      \mathcal{P}_C(\tilde{b}' | \tilde{b},a), & a \in \mathcal{C}. 
   \end{cases}
\end{equation}Here $'$ represents the corresponding states at the next time step. $\mathcal{P}_A$ defines the transitions caused by physical actions and $\mathcal{P}_C$ the transitions caused by computational actions. 
    \item The reward distribution $\mathcal{R}_{BM}:\mathcal{S}\times \mathcal{B}\times\mathcal{A}_{MB} \times \mathbb{R} \to [0,1]$ is given by 
    \begin{equation}
   \mathcal{R}_{MB} =\begin{cases} 
      \mathcal{R}_B(r|s,b,a), & a\in \mathcal{A}, \\
      \mathcal{R}_C( r|a), & a \in \mathcal{C}. 
   \end{cases}
    \end{equation}  $\mathcal{R}_B$ refers to the reward distribution of the BAMDP (see \ref{app:BAMDP_def}) and $\mathcal{R}_C$ defines the cost of performing computational actions. The usual assumption in meta-reasoning literature is to have a constant computational cost for each $a\in\mathcal{C}$, i.e. $\mathcal{R}_C = \delta(r-c)$, where $\delta$ is the Dirac-delta distribution.

    \item  $\mathcal{K}:\tilde{\mathcal{B}} \to \mathcal{Q}$ denotes a mapping from planning-beliefs to the space of action-value (i.e. $Q$) functions on the base problem, i.e. $\mathcal {Q} := \{ Q: \mathcal{S} \times \mathcal{B} \times \mathcal{A} \to \mathbb{R}\}$.  Each planning-belief in a metareasoning problem implies a value (or $Q$) function\footnote{It is crucial to note that this is the subjective value function of the agent and doesn't reflect the true value of being in a state.} which the agent can use to compare physical actions. The exact form of $\mathcal{K}$ is defined by the planning algorithm of the agent and the space $\tilde{\mathcal{B}}$.  
\end{itemize}

The goal of a meta-BAMDP agent is to find the optimal policy $\pi^*(a|s,b,\tilde b)$. However, there is some additional structure imposed on this policy. Whenever the agent chooses to take a physical action $a\in \mathcal{A}$ in a state given by $(s,b,\tilde b)$, the agent is restricted to take the ``greedy'' action according to the $Q$-function implied by the current $\tilde b$, i.e. \begin{equation} \label{eq:tree_to_action}
    a_\perp = \arg\max_{a\in \mathcal{A}} \mathcal{K}(\tilde b)(s,b,a) = \arg\max_{a\in \mathcal{A}} Q(s,b,a|\tilde b).
\end{equation} The double usage of brackets denotes that $\mathcal{K}$ is a higher-order function, whose co-domain is the space of $Q$-functions $\mathcal{Q}$\footnote{Here we have employed the notation $Q (\cdot | \tilde b)$ instead of $Q (\cdot , \tilde b)$ as the latter more naturally represents the $Q$-function of the metalevel problem whereas the former is the subjective $Q$-function (of the base problem), under a computational belief $\tilde b$.}. This structure is usually incorporated by setting $\mathcal{A} = \{\perp\}$, where $\perp$ is the \textit{terminal} action. Executing $\perp$ is equivalent to executing a physical action obtained via Eq. \ref{eq:tree_to_action}. Finally the value function corresponding to the policy $\pi$  for a meta-BAMDP is given by  \begin{equation}\label{eq:value_m-BAMDP}
\begin{split}
    V^\pi(y,t) & = \sum_a \pi(a|y,t) \bigg[ \int r\mathcal{R}_{MB}(r|s,b,a) dr  \\ &  
         + \sum_{y'} \mathcal{P}_{MB}(y' | y, a) V^\pi(y',t+1)\bigg],
\end{split}
\end{equation} along with the terminal condition $V^\pi (y,T) = 0$. Here, $y  = (s,b,\tilde b)$ and $y'$ represents the corresponding primed tuple. Note that the value function $V^\pi(y,t)$  represents the value of the metalevel problem of a (meta-)policy $\pi$, whereas $\mathcal{K}(\tilde b)$ represents the subjective value function (of the base problem) in state $y$ based on $\tilde b$. The optimal meta-BAMDP policy is given by $\pi^*(a|y,t)= \arg\max_\pi V^\pi(y,t)$. In the following section, we will see a concrete example of the definitions presented here.

\section{A meta-BAMDP for $N$-armed Bernoulli bandit task} \label{sec:m_bamdp_tabb}
While the  meta-BAMDP framework applies to a wide range of problems, we now apply it to a $N$-armed Bernoulli bandit task (NABB). In a NABB task the state space $\mathcal{S}$ is a singleton set, making the transition dynamics trivial \citep{sutton2018reinforcement}. Therefore, $\mathcal{B}$ only incorporates beliefs over reward distributions. We begin by considering how one might obtain the optimal solution to the underlying BAMDP. Assuming that the agent performs Bayesian inference and starts with a uniform prior distribution, the belief space of the NABB can be assumed to be \[
\mathcal{B} = \left\{ (\alpha_1, \beta_1, \cdots, \alpha_N,\beta_N) \in \mathbb{N}_0^{2N} : \sum_{i=1}^N \alpha_i + \beta_i \leq T \right\},
\] where $\alpha_i$ is the number of successes after taking the action $a=i$, and $\beta_i$ the failures.

\begin{figure}
	\centering
	\begin{tikzpicture}[
		scale=0.8,
		level distance=1.5cm,
		level 1/.style={sibling distance=4cm},
		level 2/.style={sibling distance=2cm},
		level 3/.style={sibling distance=1cm},
		level 4/.style={sibling distance=0.5cm},
		level 5/.style={sibling distance=0.25cm},
		every node/.style = {draw, circle, minimum size=0.1cm, inner sep=1pt},
		circle node/.style = {circle, draw, minimum size=0.1cm, inner sep=1pt}, 
		square node/.style = {rectangle, draw, minimum size=0.1cm, inner sep=1pt},
		solid edge/.style  = {edge from parent/.style={draw, ->, >=latex, shorten >=1pt}},
		dashed edge/.style = {edge from parent/.style={draw, ->, >=latex, dashed, shorten >=1pt}},
		dotted edge/.style = {edge from parent/.style={draw, ->, >=latex, dotted, shorten >=1pt}},
		edge from parent path={(\tikzparentnode.south) -- (\tikzchildnode.north)}
		]
		
		\node[circle node] {}
		child [solid edge] {node[square node] {}
			child [solid edge] {node[circle node] {}
				child [solid edge] {node[square node] {}
					child [solid edge] {node[circle node] {}}
					child [solid edge] {node[circle node] {}}
				}
				child [solid edge] {node[square node] {}
					child [solid edge] {node[circle node] {}}
					child [solid edge] {node[circle node] {}}
				}
			}
			child [solid edge] {node[circle node] {}
				child [solid edge] {node[square node] {}
					child [solid edge] {node[circle node] {}}
					child [solid edge] {node[circle node] {}}
				}
				child [solid edge] {node[square node] {}
					child [solid edge] {node[circle node] {}}
					child [solid edge] {node[circle node] {}}
				}
			}
		}
		child [solid edge] {node[square node] {}
			child [solid edge] {node[circle node] {}
				child [dotted edge] {node[square node] {}
					child [dotted edge] {node[circle node] {}}
					child [dotted edge] {node[circle node] {}}
				}
				child [dotted edge] {node[square node] {}
					child [dotted edge] {node[circle node] {}}
					child [dotted edge] {node[circle node] {}}
				}
			}
			child [solid edge] {node[circle node] {}
				child [dashed edge] {node[square node] {}
					child [dashed edge] {node[circle node] {}}
					child [dashed edge] {node[circle node] {}}
				}
				child [dotted edge] {node[square node] {}
					child [dotted edge] {node[circle node] {}}
					child [dotted edge] {node[circle node] {}}
				}
			}
		}
		; 
		
	\end{tikzpicture}
	
	\caption{Planning-belief as a subgraph over possible future states in a 2-armed Bernoulli bandit.
		Illustrates the structure of the belief-action DAG constructed during planning. The solid portion represents the current planning-belief $\tilde{b}$, the dotted area is the unexplored portion of the full graph, and the dashed edge indicates a potential node expansion (computational action). Terminal states (circles) receive heuristic values, and backward induction propagates these values to earlier belief nodes.}
	\label{fig:decision_action_tree}
\end{figure}

In order to obtain the optimal BAMDP policy, the agent can be imagined to construct the complete belief-action graph (see Fig. \ref{fig:decision_action_tree} for a schematic with $N=2$) associated with NABB. The value of the terminal beliefs (circles in Fig. \ref{fig:decision_action_tree}) can be set to zero and then iteratively the values of all the non-terminal beliefs (and actions) can be obtained via backward induction until the value of the root node $V^\pi (\bm b_0)$ is found. Therefore, if given access to the entire graph, the agent may make use of backward induction to find the policy that maximizes $V^\pi (\bm b_0)$.

But what action should an agent take, if it does not have access to the entire graph? This would be akin to a situation where the agent has not considered the consequences of all of its actions until the end of the task. Let us say, that the agent only has access to a sub-graph (in solid lines in Fig. \ref{fig:decision_action_tree}). This sub-graph represents the agent's planning-belief or $\tilde b$\footnote{Representing planning-beliefs as subgraphs has a natural history \citep{callaway2022rational,huys2012bonsai} when modeling human behavior.}.    The agent can then be assumed to perform backward induction only on $\tilde b$, for a given value of the terminal nodes of $\tilde b$. But how would an agent determine the values of the terminal nodes without knowing the subsequent graph? Here is a crucial assumption that a cognitive scientist needs to make. \citep{callaway2022rational} chose a random strategy. We make an alternate choice that is more natural choice for bandit settings - we assume the terminal values derive from the actions (hypothetically) being purely exploitative, thereafter, i.e. taking the physical action with the highest expected reward until the horizon. This was also previously used in the knowledge gradient (KG) policy \citep{frazier2008}. Concretely, the value $U$ of a terminal node
$\bm b = (\{\alpha_i, \beta_i\}_1^N)$ is given by 
\begin{equation} \label{eq:greedy_act}
    U(\bm b) = \max \Big(\Big\{\frac{\alpha_i+1}{\alpha_i+\beta_i+2}\Big\}_1^N\Big)\tau,
\end{equation}
where $\tau = T-\sum_i (\alpha_i +\beta_i)$, is the remaining rounds of the NABB. This assumption essentially says that, from the perspective of the agent, it is going to stick to the greedy action from $\bm b$ until the horizon, and while it does so, it will not learn and update its beliefs.

With the terminal values of $\tilde b$ in place, we can proceed with obtaining the values of preceding action nodes (squares in Fig. \ref{fig:decision_action_tree}) giving the subjective $Q$-values ($\mathcal{K}$ from Sec. \ref{sec:m-BAMDP-defn}). The maximization over the $Q$-values provides the subjective values of the preceding beliefs (circles in Fig. \ref{fig:decision_action_tree}), and so on, giving us the subjective values of all beliefs and actions\footnote{Note that $K(\tilde b)$ gives us the optimal BAMDP $Q$-function when $\tilde b$ is the complete DAG and the greedy $Q$-function when $\tilde b$ is just the root node.}. From this subjective value $\mathcal{K}(\tilde b)$ we can get the terminal action to be taken in the root node $\bm b_0$, as obtained from Eq. \ref{eq:tree_to_action}. 

In addition to taking the currently best physical action, the agent can instead take a computational action. A natural choice for the computational action, which we assume for the remainder of the paper, is \textit{node expansion} whereby the agent can expand the current graph $\tilde b$ and obtain a new graph $\tilde b'$ , by adding an action node and its child states (see Fig. \ref{fig:decision_action_tree}). We assume the cost of taking each node expansion action is $c$, encompassing costs associated with time, energy, opportunity cost (e.g. given constrained attention or working memory capacity in humans), etc..

Finally, combining these two types of actions, the meta-policy $\pi$ of the agent can be viewed as a mapping from $\pi : \mathcal{B} \times \tilde{\mathcal{B}} \times  (\{\perp\} \cup\mathcal{C}) \to [0,1]$\footnote{We have dropped the explicit temporal dependence as in the bandit setting $\bm b$ captures that information.}. The optimal meta policy  $\pi^*$ can then be found using backward induction on the \textit{meta-graph} where the nodes correspond to the states $(\bm b, \tilde b)$ of the meta-BAMDP - i.e. by solving the Bellman equation,\begin{equation}\label{eq:meta_value_decomp}\begin{split}
    V^\pi(\bm b,\tilde b,t) &= \pi(a=\perp | \bm b,\tilde b) \Big[\int r\mathcal{R}_B(r|a=\perp,\bm b)dr  \\ & + \sum_{\bm b'} V^\pi(\bm b',\tilde b,t+1) \mathcal{P}_B (\bm b'| \bm b,\tilde b,a=\perp )\Big]  \\ & +  \sum_{a \neq\perp} \pi(a  | \bm b,\tilde b)  \Big[ -c + \sum_{\tilde b'} V^\pi(\bm b,\tilde b',t)\mathcal{P}_C (\tilde b'| \tilde b,a ) \Big],
\end{split} 
\end{equation} with the boundary condition $V^\pi(\bm b,\tilde b,T) = 0$. Here we have explicitly expressed the contributions from physical and computational actions. More specifically, in the last summand, note that $V^\pi$ is evaluated at time $t$ and not $t+1$. This makes explicit our notion of time - it is primarily an index for the number of physical interactions with the environment. In a more realistic setting one might wish to also consider the time needed to perform computations. In such a case, one would need to consider a semi-(BA)MDP \citep{SUTTON1999}, instead of a BAMDP (as there can be multiple thinking iterations between two consecutive updates to $\bm b$), but this is out of scope for our present work.

While simple in principle, the size of the state space of the meta-BAMDP explodes exponentially (in both $T$ and $N$). This is true because the number of subgraphs of a graph $G$ grows exponentially in the number of edges in $G$, which in this case itself grows polynomially in $T$ and exponentially in $N$ (see \ref{app:complexity_concerns} for a detailed discussion). This means that using backward induction to recursively solve the Bellman equation in Eq. \ref{eq:value_m-BAMDP} to find $\pi^*$ is hopelessly infeasible. In the following section we exploit some regularities of the meta-BAMDP in Bernoulli Bandit tasks to prune the meta-graph, and subsequently find good approximations to the solution.

\section{Finding good approximations via pruning} \label{sec:Approximate_solns}
 We now make a series of arguments geared toward pruning the meta-graph. These arguments take the form of two theorems and their corollaries (with the complete proofs in \ref{app:proofs}) which apply to $N$-armed Bernoulli Bandit (NABB) tasks. Below each theorem we present a few intuitive remarks.

\begin{theorem} \label{Thm:MC}
The optimal meta-policy $\pi^*$ is a \textit{"mind-changer"}. I.e. if for any starting state $(\bm b, \tilde b)$, $\pi^*$ prescribes performing computations till $(\bm b, \tilde b^\prime)$ and then terminate (i.e. chooses only $a\in \mathcal{C}$ until reaching $\tilde b^\prime$, then chooses $\perp$), then either of the following is true. \begin{enumerate}
    \item $\tilde b^\prime \neq \tilde b$ and $a_\perp(\bm b, \tilde b) \neq a_\perp(\bm b, \tilde b^\prime)$, or
    \item $\tilde b^\prime = \tilde b$,
\end{enumerate} where $a_\perp(\cdot)$ represents the terminal action in state $(\cdot)$, and is obtained from the subjective value function $\mathcal{K}$ as in Eq. \ref{eq:tree_to_action}.
\end{theorem}

This theorem states an important property of the optimal meta-policy $\pi^*$ - if further computations don’t alter the agent’s decision at the current time-step, then performing those computations is wasteful. This is true as one doesn't gain anything by computing for future states already. Intuitively, as most future states are not going to be visited, performing computations in the present state that will prove to be useful only in those specific future states, would be wasteful (in expectation).

\begin{corollary} \label{Thm:MMC}
    The optimal meta-policy is a \textit{minimal mind-changer}, i.e. if the optimal meta policy $\pi^*$ prescribes computation till $(\bm b,\tilde b^\prime)$ from a starting state $(\bm b,\tilde b)$ then $a_\perp(\bm b,\tilde b^\prime) \neq a_\perp(\bm b,x)$, for all subgraphs $x$, where  $\tilde b \subseteq x \subset \tilde b^\prime$.  
\end{corollary}

This directly follows from repeated application of Thm. \ref{Thm:MC} along any given computational trajectory. Therefore, once we find a minimal mind-changing computational trajectory on the meta-graph, we need not look at the downstream nodes in the meta-graph.

\begin{theorem} \label{Thm:monotonicity}
    {Computation (not strictly) monotonically increases subjective value $\mathcal{K}(\tilde b)(\bm b, a)$, i.e. $\mathcal{K}(\tilde b)(\bm b, a) \leq \mathcal{K}(\tilde b^\prime)(\bm b, a)$, $\forall \bm b$ iff. $\tilde b\subseteq \tilde b^\prime$.}
\end{theorem}

This is a crucial property needed to terminate the search for minimal mind-changing meta-policies.  Here we consider how the agent's subjective value changes as a function of a single computational action. In particular, the theorem states that the subjective value (for the case of $N$ABB) only monotonically increases as the agent computes. This implies that the agent's subjective value is upper bounded by the optimal value of being in the state, as implied by the optimal policy of the underlying BAMDP (value under the complete DAG). Further, as we choose the greedy heuristic for the terminal condition (Eq. \ref{eq:greedy_act}), the lower bound of the subjective value is given by the value of the terminal heuristic (value under the singleton DAG\footnote{I.e. the trivial DAG with just a single (root) node.}).

\begin{corollary} \label{Thm:Termination}
    $\pi^*$ prescribes termination in all states $(\bm b, \tilde b)$ if $\exists $ a physical action $ i$, such that $Q(i,\bm b|\tilde b) \geq Q^*(j,\bm b)$, $\forall j \neq i$ \footnote{Note the subtle difference in the notation between the subjective $Q (\cdot | \cdot)$ and the optimal (for the underlying BAMDP) $Q^* (\cdot)$ functions. Both of these represent the $Q$-functions of the base problem.}.
\end{corollary} 

Due to Thm. \ref{Thm:monotonicity}, $Q^* (j,\bm b)$ is the upper bound of $Q(j,\bm b| \tilde b)$, $\forall j$. If in a state $(\bm b,\tilde b)$, the subjective $Q$-value corresponding to an arm $i$ is greater than the upper bound of the $Q$-values of the other arms $j\neq i$,  then, no matter how many computations the agent performs, its terminal action will always remain the same. From the minimal mind-changing property of Thm. \ref{Thm:MC} and Corollary \ref{Thm:MMC} recall that if the agent's terminal action cannot change, then it must terminate. 

\begin{definition}[$\mathcal{M}$-beliefs]
    $\mathcal{M}$ is the set of all beliefs $\bm b$ such that $Q(i,\bm b|\tilde b) \geq Q^*(j,\bm b), \forall j \neq i$, where $\tilde b$ is a singleton graph with the node $\bm b$. Here $i$ is the greedy arm and $j \neq i$ is a non-greedy arm.
\end{definition} 

\begin{corollary} \label{Thm:M_state_computaion}
    $\pi^*$ never prescribes computation along the non-greedy arm for a belief $\bm b \in\mathcal{M}$.
\end{corollary} 
This follows a similar reasoning as the previous corollary.

Lastly, we take note of the fact that the increase in the subjective value $Q$ decreases geometrically with computation depth: $P(\bm b' \to \bm b'')$ includes the product of transition probabilities starting from $\bm b'$ to $\bm b''$, which only decreases geometrically with the temporal distance between $\bm b'$ and $\bm b''$. In other words, the return on computation diminishes geometrically. This suggests that it might be worthwhile to bound the maximum size of a planning-belief $\tilde b$ to be $|\tilde b |\leq k$, $|\tilde b|$ refers to the number of edges in $\tilde b$. Since one computational action introduces two edges, we start by assuming $k=2$, i.e. allowing at most one computational action in each step (i.e. a meta-myopic policy). We observe that our results are invariant for $2\leq k \leq 16$ and also for alternate approximation schemes (see \ref{app:robust_approximations} for further details), indicating that for our problem, setting a low bound on reasoning steps is near-optimal. Guided by these considerations, we build an algorithm (see \ref{app:pseudocode}) to find the solution to the meta-BAMDP. We proceed with analyzing the solutions and their implications for human behavior in TABB tasks.

\begin{figure*}
\centering
\subfloat[]{\includegraphics[width = 1.6in]{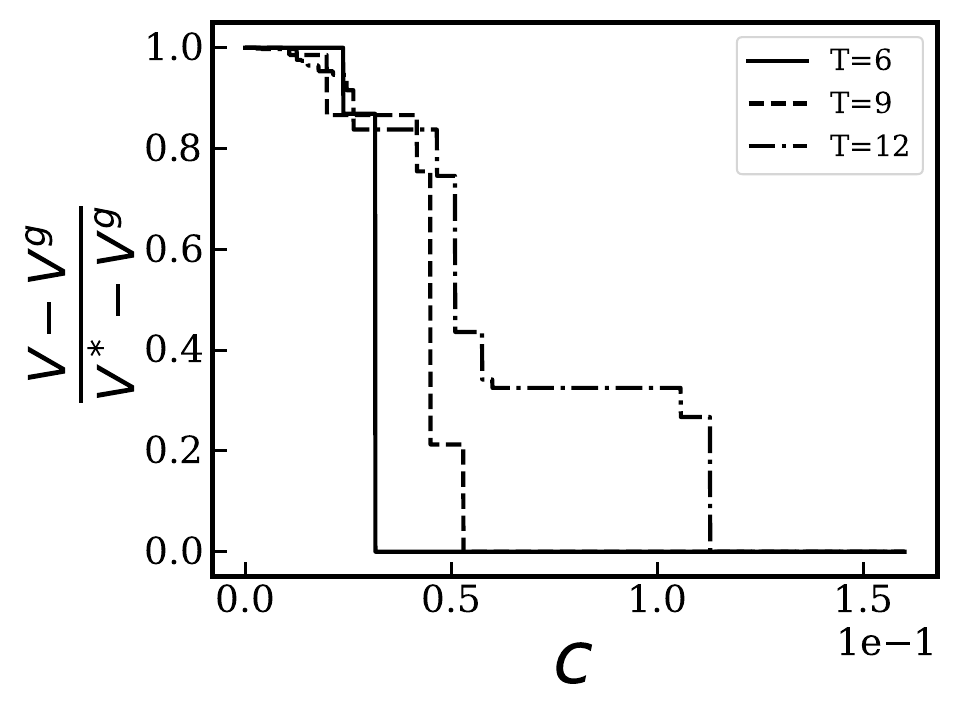}} 
\subfloat[]{\includegraphics[width = 1.6in]{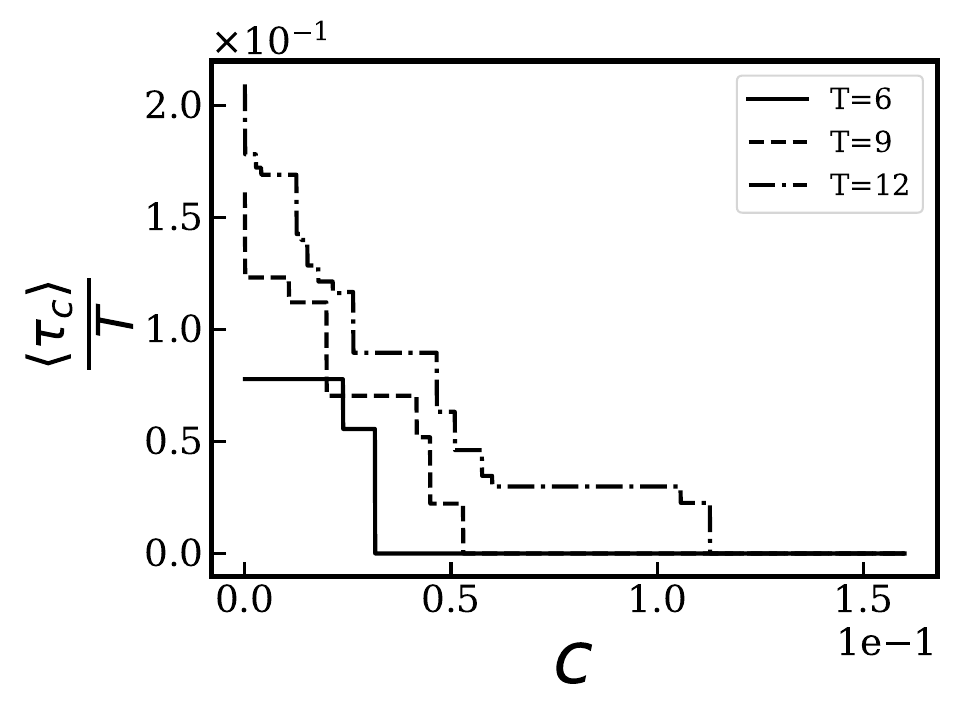}}
\subfloat[]{\includegraphics[width = 1.6in]{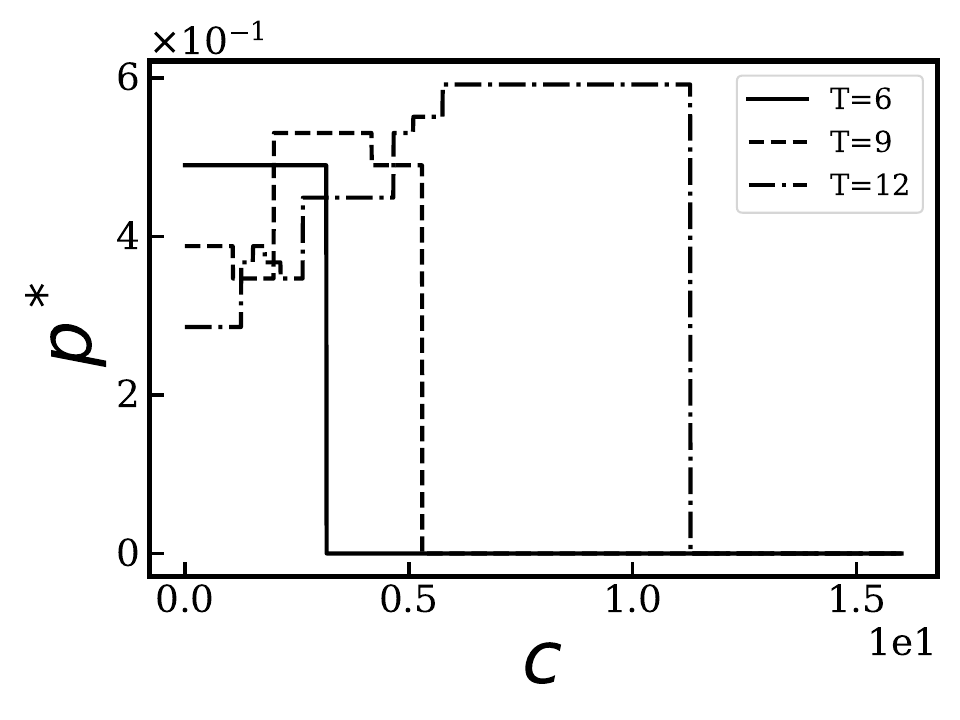}}
\caption{Resource-rational behavior under computational cost constraints. (a) Normalized, total expected reward accrued under the optimal meta-policy for a given computational cost and different task lengths. (b) Average time-step (in the TABB task) at which a node-expansion action is performed, as a function of the computational cost, for tasks of different lengths.  (c) Environments in which most computations are performed as a function of computational cost, for different task lengths.}
\label{fig:literature_match}
\end{figure*}
\section{Implications for human exploration behavior in TABB tasks}\label{sec:predictions_human_behavior}

 We discuss below the solutions of the meta-BAMDP for varying computational costs $c$ for TABB tasks ($N=2$), of different lengths. We consider TABB because this setup is one of the most studied in the literature. Additionally, the behavior for varying number of arms is not qualitatively different than the $N=2$ case\footnote{Refer to \ref{app:number_of_arms} for a discussion on the few points of differences with higher number of arms.}. We begin by presenting results for a uniform distribution over all TABB environments, where a TABB environment is given by a tuple $(p_1,p_2)\in [0,1]\times[0,1]$, of stationary reward probabilities. The initial state we consider for all meta-optimal agents is $(\bm b, \tilde b) = (\bm 0, \tilde 0)$, where $\bm 0$ is the belief where $\alpha_i = \beta_i=0$ for all arms $i$ and $\tilde 0$ is the singleton sub-DAG with just the root node $\bm 0$. In addition to the analyses presented below we also compare the meta-optimal policies to knowledge gradient (KG) policies. For this, interested readers are referred to \ref{app:KG_policies}.

When comparing behavior in tasks of different lengths, we first consider the normalized expected reward (i.e. ignoring the computational costs) under the optimal meta-policy as a function of computational cost (c.f. Fig. \ref{fig:literature_match}(a)).   From  Eq. \ref{eq:meta_value_decomp} it is evident that the external rewards accrued by a meta-policy are lower bounded by the value of the greedy policy $V^g$ and upper bounded by the value of the optimal BAMDP policy, $V^*$. Therefore we define the normalized expected reward as $V^N = \frac{V-V^g}{V^* - V^g}$. As one would intuit, we observe from Fig. \ref{fig:literature_match}(a) that $V^N$ monotonically decreases with $c$. This offers a novel computational explanation of the positive correlation observed between IQ and bandit task performance \citep{STEYVERS2009}, i.e\ individuals with higher working memory and attentional capacity have lower  computational cost $c$ and therefore are able to plan further ahead and make better decisions. Additionally, we observe a characteristic dependence of behavior on the task horizon $T$. For shorter tasks, we find that the discrete jumps in the normalized value are higher, i.e. that behavior is either near-optimal or near-greedy. A systematic exploration of the dependence of average accrued reward on $T$ should be able to test such behavior.

Beyond the average performance, the meta-BAMDP also allows us to estimate the conditions under which an agent performs a computation. We consider $\langle \tau_c\rangle$, the average time step at which a computational action is performed. In Fig. \ref{fig:literature_match}(b) we see that as computational cost decreases, people explore until later in the task. Here we plot $\langle \tau_c \rangle$ normalized by the task horizon $T$. For high computational costs, agents perform only a few computations and they do so earlier on in the task. Such behavior is likely observable by testing the reaction-times (and its dependence of task iterate $t$) of subjects. However, we do acknowledge that inferring computational actions from experiments is non-trivial, therefore, one might instead focus on the behavioral consequence of being able to compute, i.e. \textit{exploratory} actions, where an action is considered to be exploratory if it chooses the less rewarding arm based on the belief $\bm b$, and if both the arms are equally rewarding, it chooses the (thus far) less chosen arm.

 We have seen \textit{when} (in relation to the task horizon) meta-optimal agents compute. We now consider in \textit{which} environments the agents compute. As in asymmetric environments ($p_1\neq p_2$) the agent is more likely to reach $\mathcal{M}$ beliefs (and therefore not compute) than in symmetric environments ($p_1=p_2 = p$), we only consider symmetric environments. In Fig. \ref{fig:literature_match}(c) we plot the environment $p^*$ where the agent performs the most computations given a computational cost $c$. Quite interestingly, we observe qualitatively different behavior for high and low computational costs. We find that, for low computational costs, agents compute the most in reward scarce environments. On the other hand, for high computational costs, agents compute in reward rich environments.

\subsection{Mapping to experimental data}

We now compare the the behavior induced by the solutions of the meta-BAMDP and human behavior in Bandit tasks. It has been reported experimentally that there are majorly two ways in which human exploration behavior adapts to computational constraints - a reduction in \textit{directed exploration} with computational constraints (eg - time pressure, working memory loading, etc.) \citep{otto2014physiological,cogliati2019should}  and increase in choice repetition or decrease in action entropy with computational constraints \citep{wu2022}. We will show how both these qualitative features emerge from our model, without making any ad-hoc assumptions about human behavior.

\begin{figure}
\subfloat[]{\includegraphics[width = 1.6in]{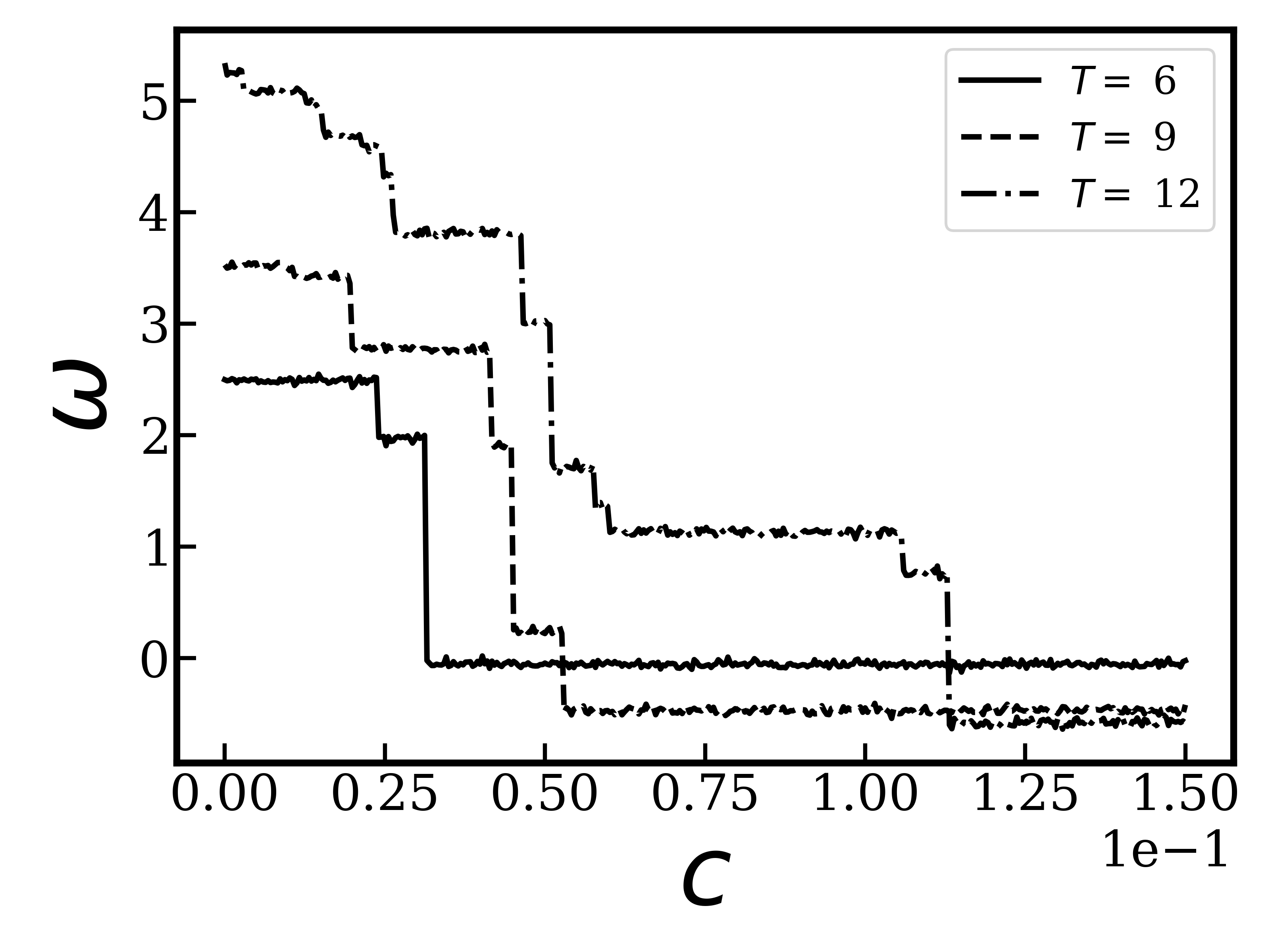}}
\subfloat[]{\includegraphics[width = 1.6in]{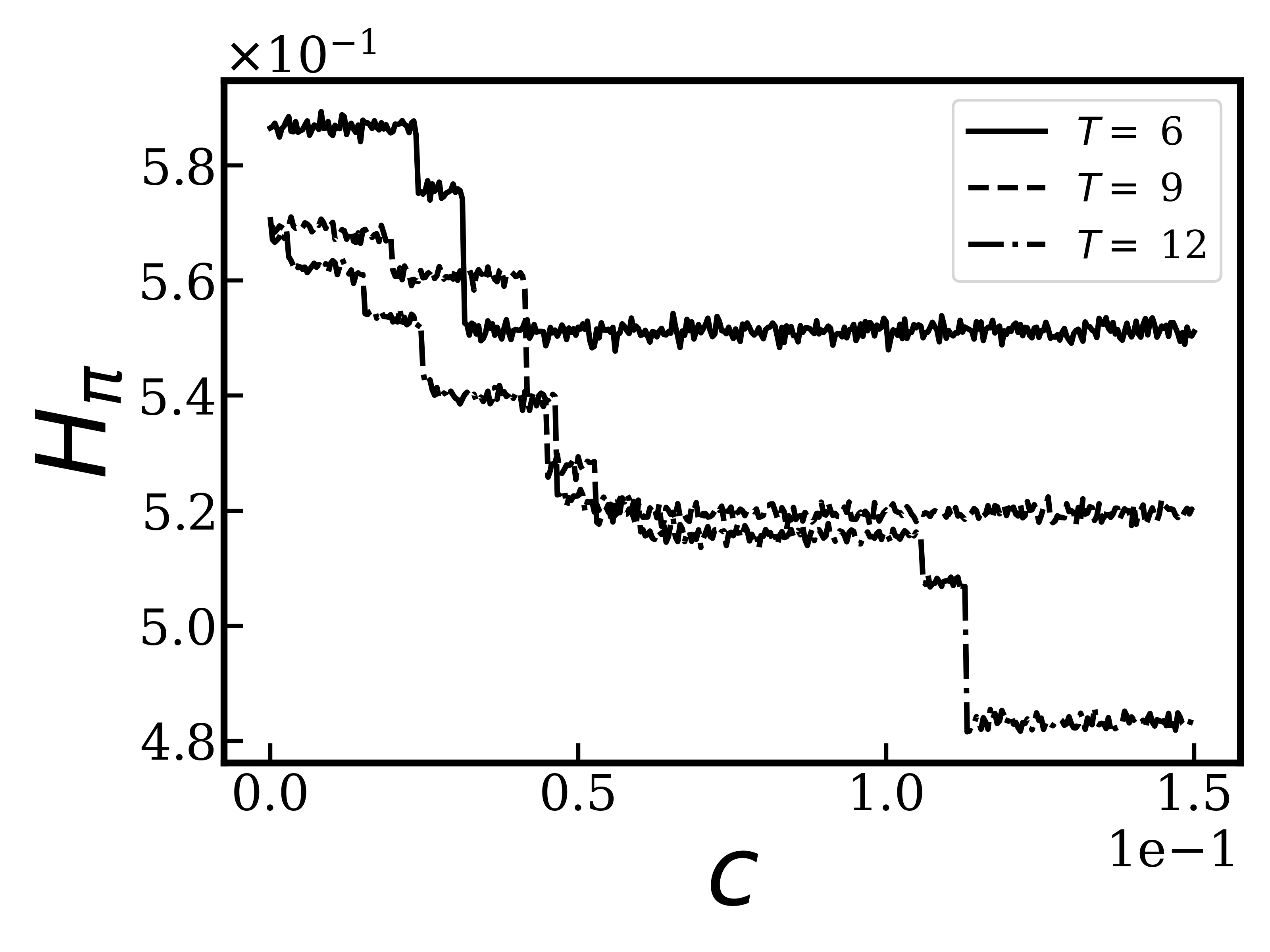}}
    \caption{Meta-BAMDP reproduces key signatures of human adaptation to cognitive load.
(a) As computational cost $c$ increases, the fitted uncertainty-bonus parameter (reflecting directed exploration) decreases. This matches experimental findings that cognitive load reduces directed exploration. The effect is consistent across task lengths.
(b) Action entropy (diversity of choices) also decreases with increasing computational cost, with longer-horizon tasks exhibiting lower entropy overall. This reflects increased choice repetition and aligns with recent empirical observations in humans.}
    \label{fig:experimental_map}
\end{figure}

First, we consider the heuristic strategy from \citep{BROWN2022}, where an agent takes actions according to $\pi(a=i|\bm b) \propto e^{-(\beta \hat \mu_i + \omega \hat \sigma_i)}$. Here $\hat \mu_i, \hat \sigma_i$ are the estimated mean and variance of the belief distribution corresponding to state $\bm b$, i.e. a soft-max policy based on a linear combination of expected reward and an ``uncertainty bonus''. The $\omega$ parameter would be our formalization of \textit{directed exploration}. We fit this policy to the behavior induced by the optimal meta-policy, in a symmetric environment with $p_1=p_2=0.5$, and plot the corresponding value of $\omega$ averaged over $10^5$ simulation runs (see Fig. \ref{fig:experimental_map}(a)), the uncertainty bonus (for details see \ref{app:heuristics_fit_procedure}). We observe that, as computational costs increase, uncertainty driven exploration decreases, irrespective of task horizon. Additionally, we also find that generally, increasing $T$ also increases $\omega$ (for a fixed computational cost), which also matches well with the experimental observations \citep{wilson2014,BROWN2022}.  That directed exploration reduces as a function of computational costs, seems to suggest that directed exploration is computationally costlier than exploitation. In light of the meta-BAMDP framework this further suggests that humans indeed have a greedy terminal heuristic and that computations are needed to be able to deviate from the baseline greedy strategy (towards optimality).

Second, we consider variations in action entropy\footnote{Here we consider the unconditional action entropy, which is simply the Shannon entropy of the histogram of actions taken in the task.} $H_\pi$ in response to variations in computational costs $c$. We observe that our meta-optimal agents demonstrate the same behavior as observed by \citep{wu2022}. In Fig. \ref{fig:experimental_map}(b) we show the action entropy $H_\pi$ (averaged over $10^5$ simulation runs) for meta-optimal agents in the symmetric environment with $p_1=p_2=0.5$. We observe that as computational costs increase, the action entropy decreases for games of all lengths. Additionally, for longer games, the action entropy is even lower, because of more available time to repeat actions. 

 The above maps to experimentally observed behavior demonstrate the power of the meta-BAMDP framework. Specifically, that such a simple model with only modest assumptions about computations is able to capture crucial aspects of human learning and exploration behavior in Bandit tasks. We now proceed to make further predictions about human behavior in bandit tasks.

\subsection{Sensitivity to computational cost manipulations }
Averaging over all the environments necessarily loses information about how subjects might adapt to specific environment statistics. We therefore, now proceed with exploring environment specific behavior. In particular, we are interested in how the variations of computational costs (time pressure, burdening working memory, etc.) impact the behavior. I.e., in which environments is the experimenter likely to observe a statistically significant response to experimental manipulations (of the computational cost)? 

For this we define two new quantities $\chi_\tau$ and $\chi_V$ which denote the \textit{sensitivity} of $\tau,V$ respectively, to changes in computational costs. More precisely, the sensitivity of an observable $X$ is defined to be \begin{equation}
    \chi_X =  \int dc\big(\frac{d  X }{d c}\big)^2.
\end{equation}

\begin{figure*}
\centering
\subfloat[$T=4$]{\includegraphics[width = 1.6in]{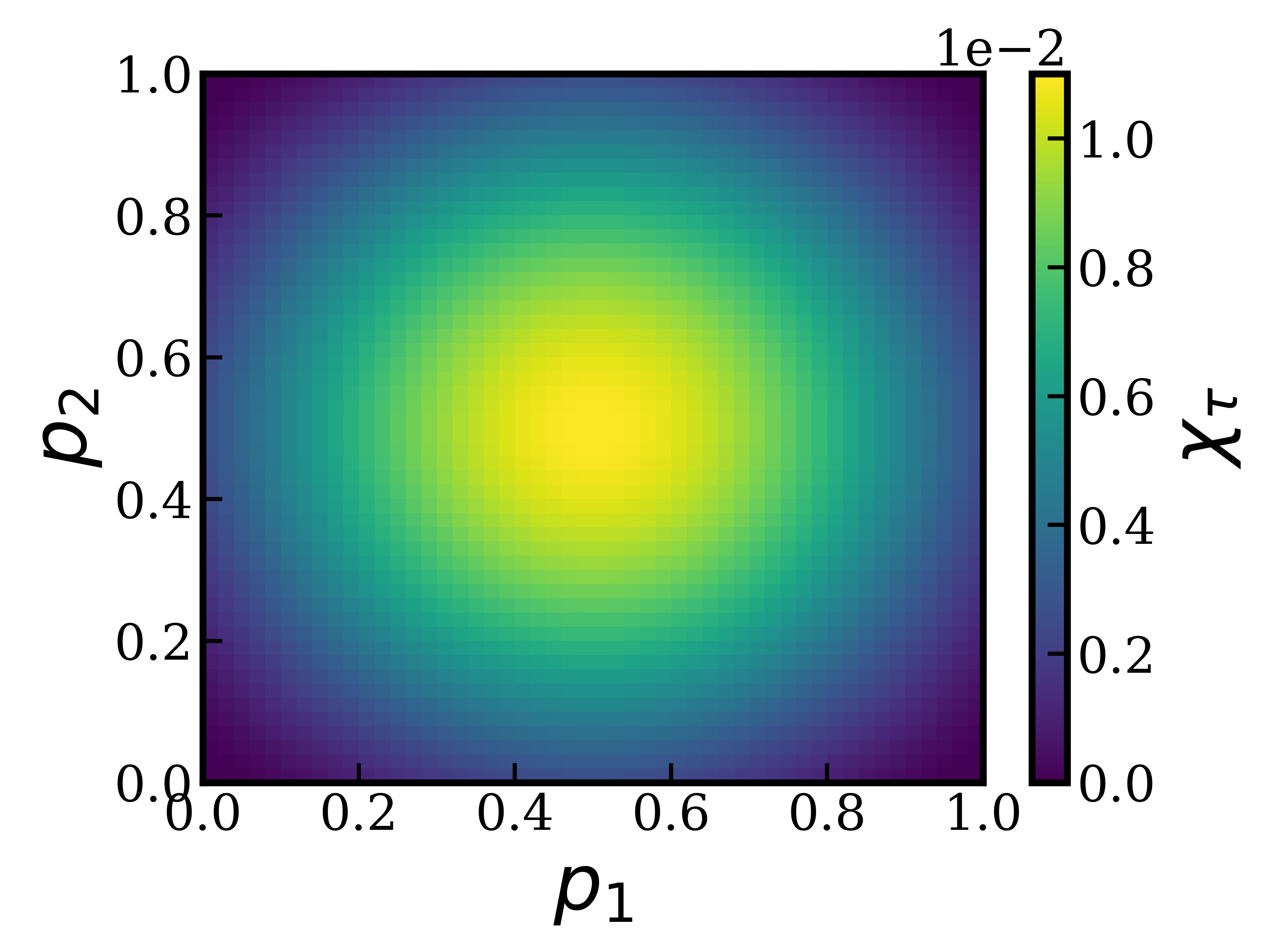}} 
\subfloat[$T=8$]{\includegraphics[width = 1.6in]{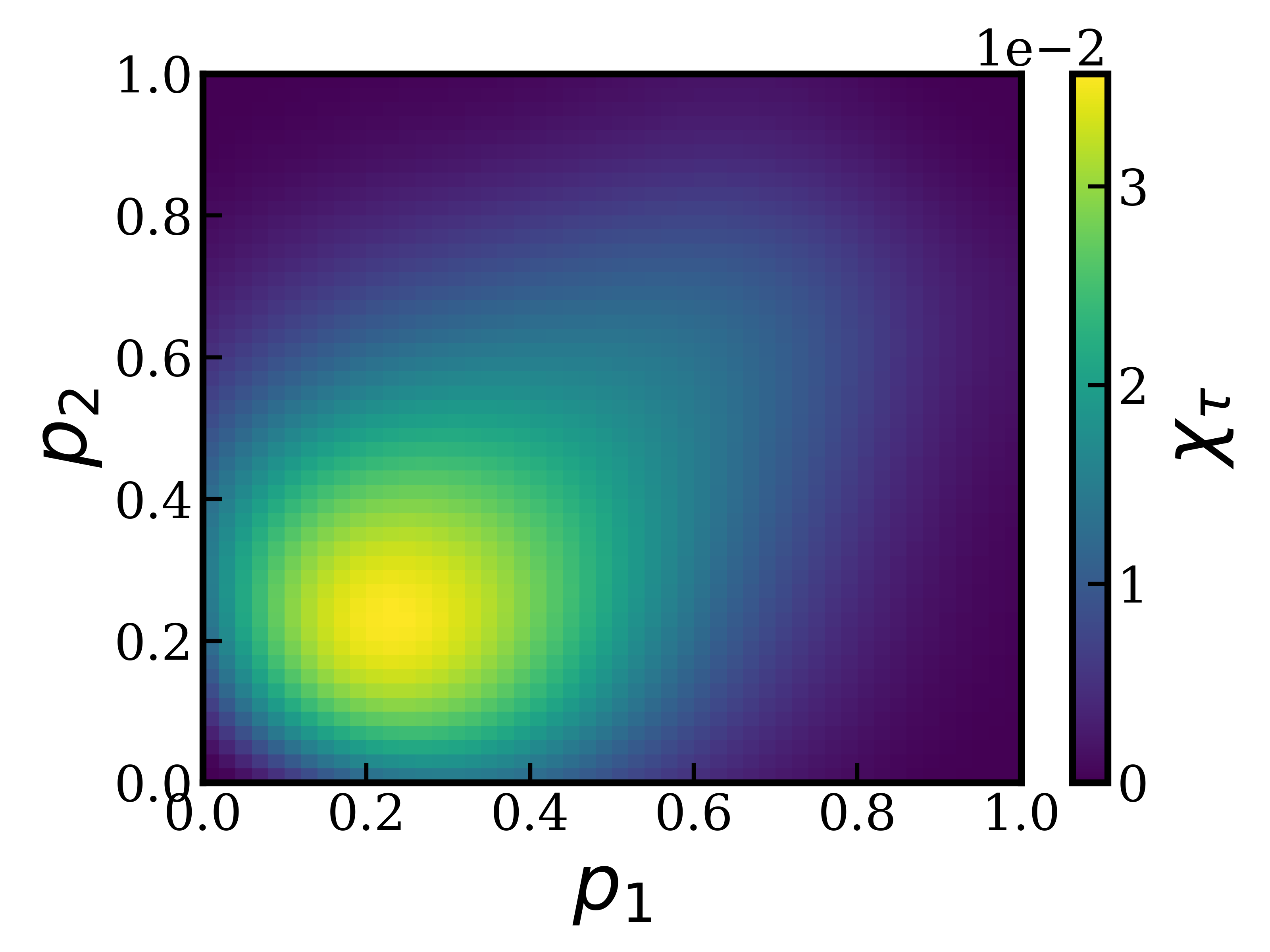}}
\subfloat[$T=12$]{\includegraphics[width = 1.6in]{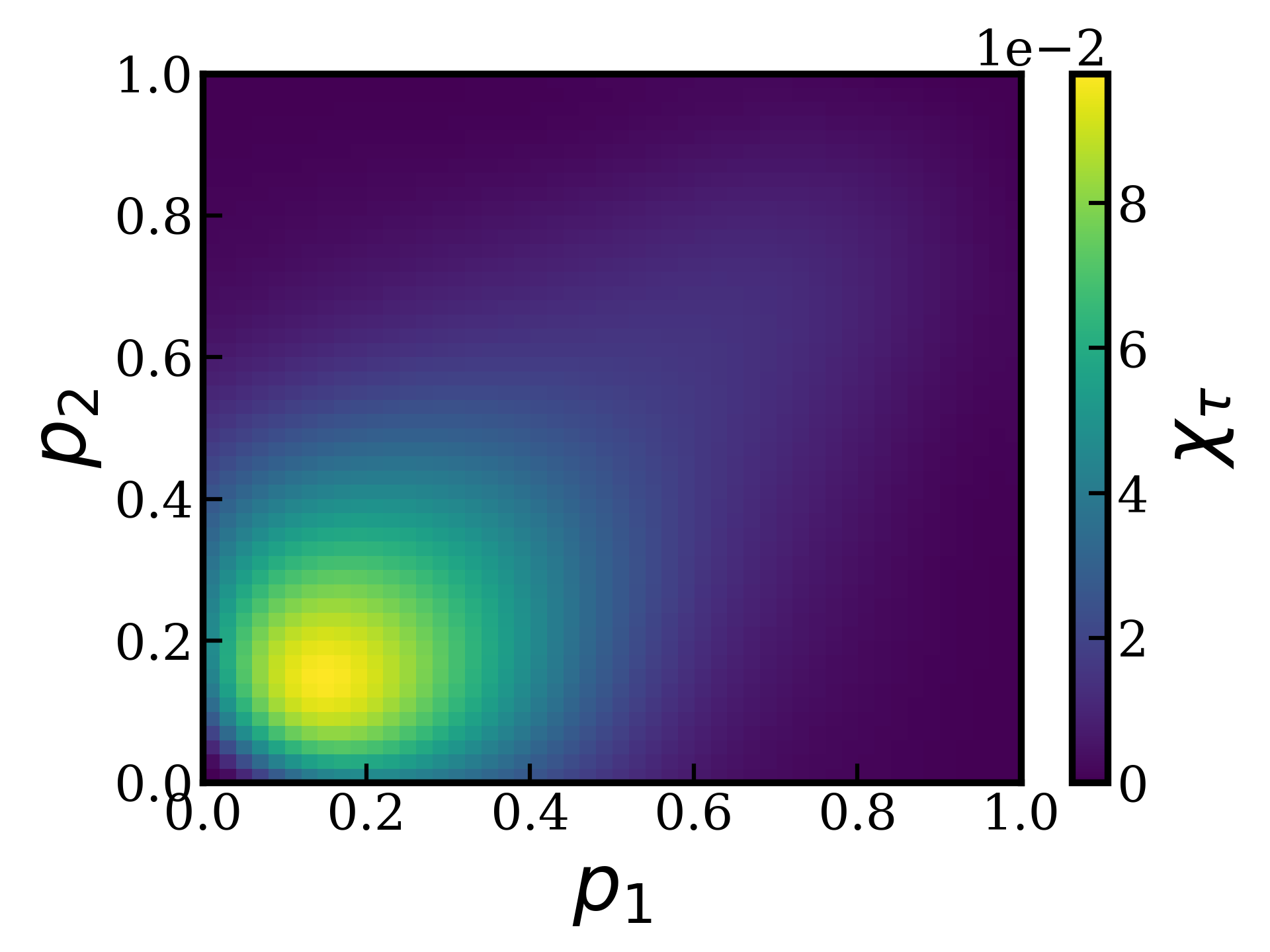}}\\
\subfloat[$T=4$]{\includegraphics[width = 1.6in]{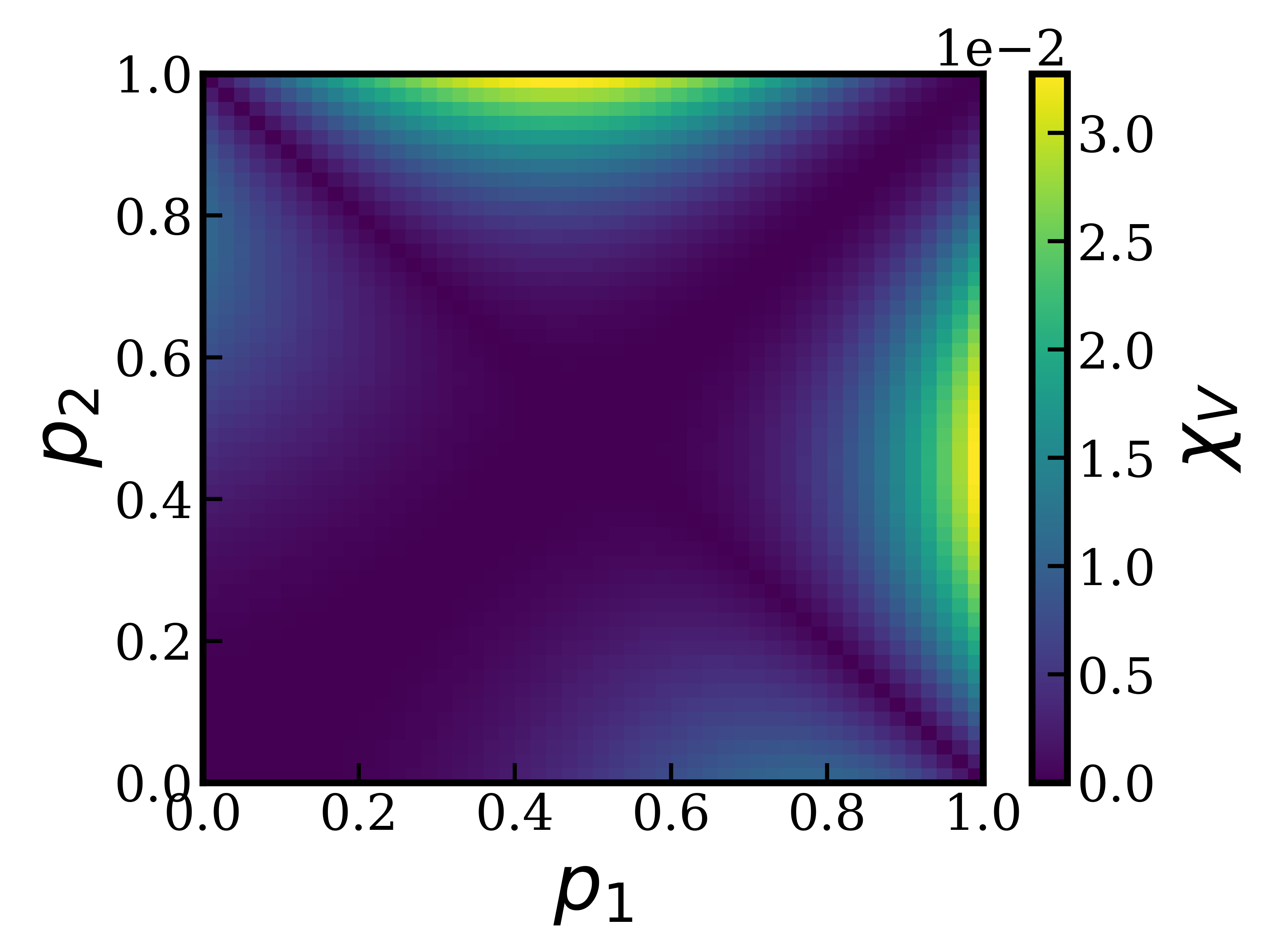}} 
\subfloat[$T=8$]{\includegraphics[width = 1.6in]{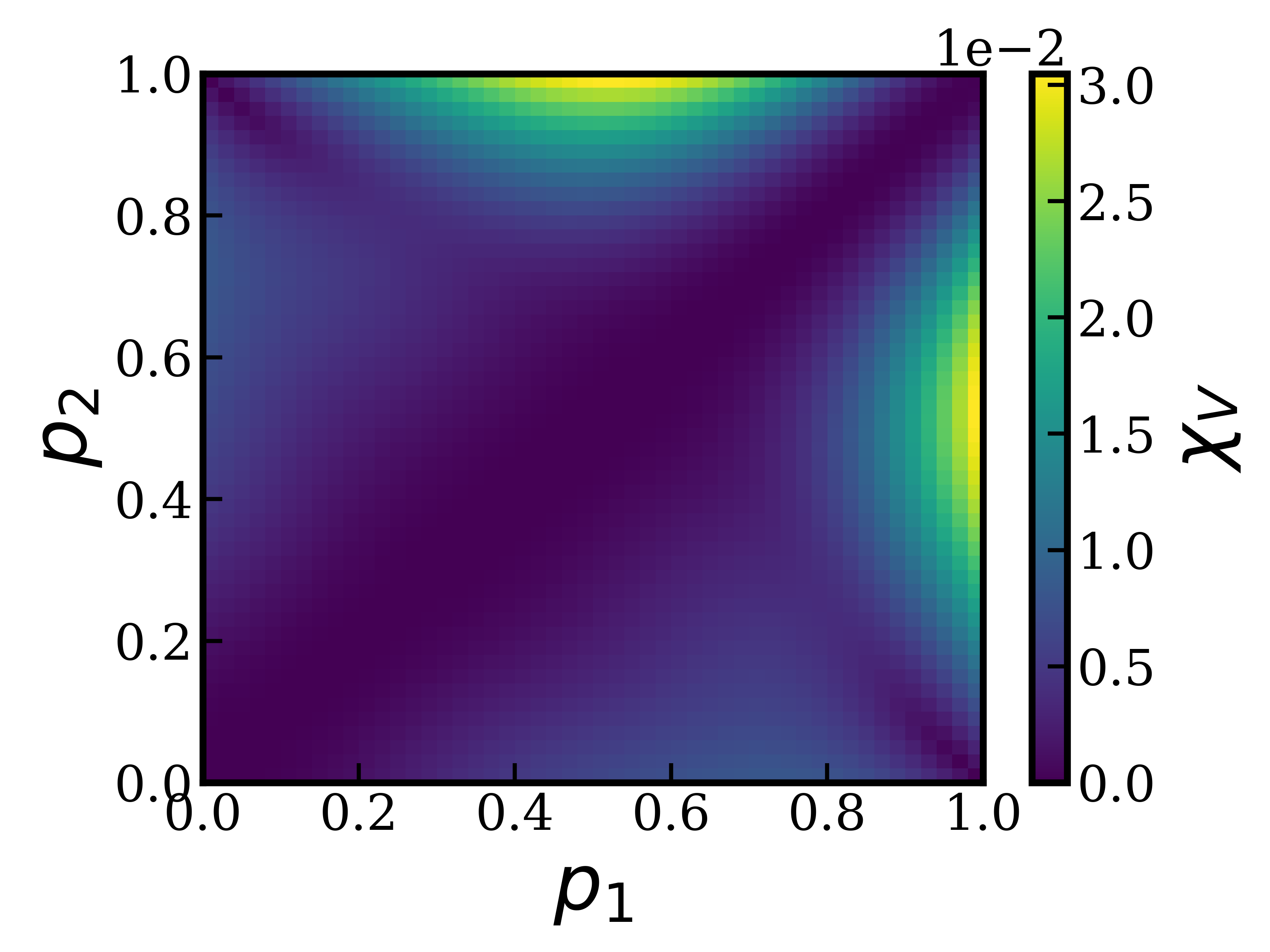}}
\subfloat[$T=12$]{\includegraphics[width = 1.6in]{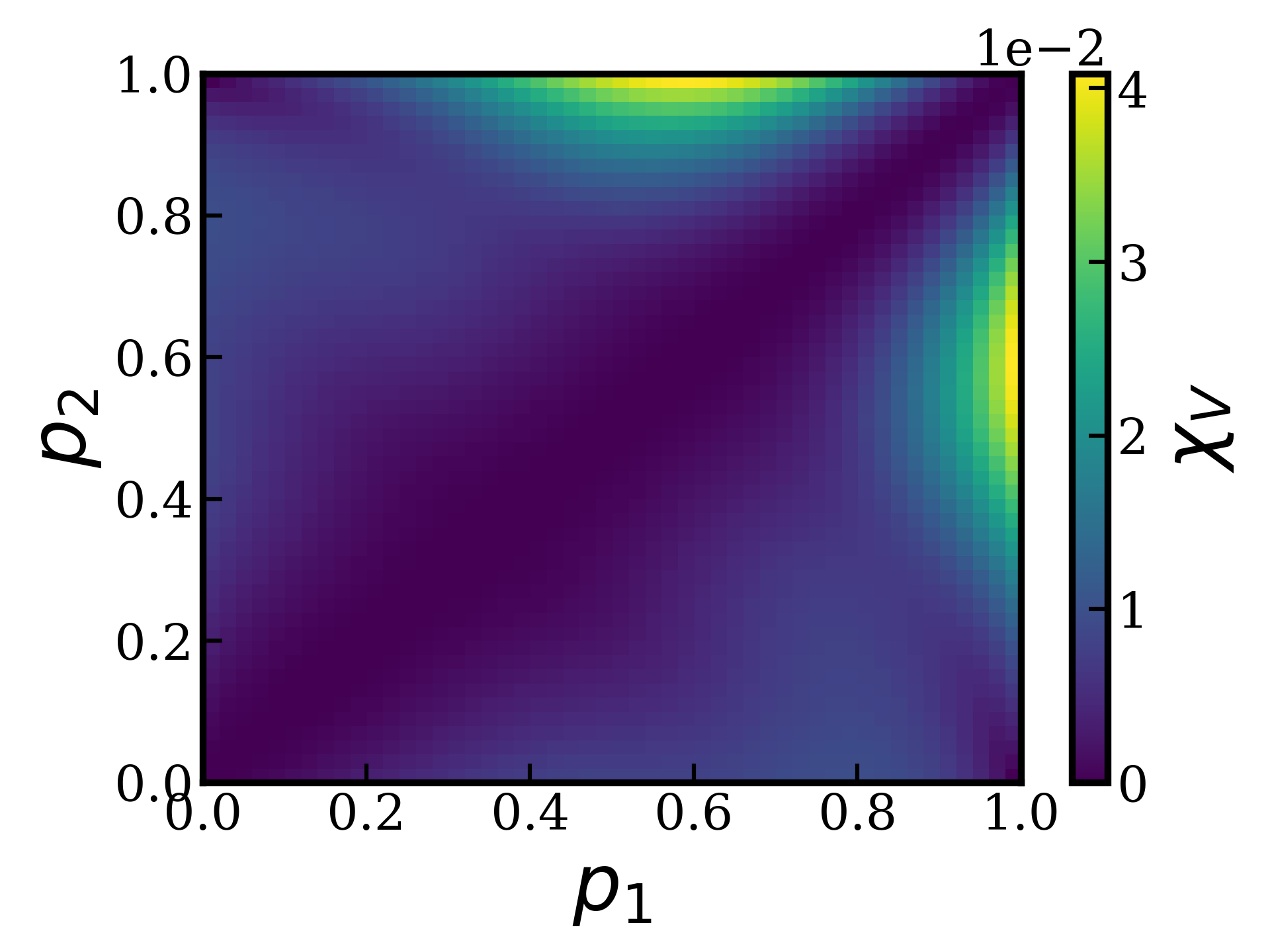}}
\caption{Environment-dependent sensitivity of exploration and performance to cognitive cost manipulations.(a,b,c) Sensitivity of average time at which exploratory actions are taken, to changes in computational cost, for different task/horizons lengths ($T=4,8,12$). Agents in sparse-reward environments ($p_1,p_2<0.5$) are more sensitive to cost, suggesting greater variance in exploration timing under cognitive load.(d,e,f) Sensitivity of total expected reward to changes in computational cost, for different task lengths/horizons. Greatest effects are seen in reward-rich environments with moderate asymmetry, where strategic planning provides the largest benefit. These results predict where manipulations (e.g., time pressure, WM load) will yield largest behavioral effects.}
\label{fig:predictions}
\end{figure*}

The observables of our interest are $\tau$, which represents the average time step at which the agent performs an exploratory action (as defined above), and $V$, which is the total expected reward obtained by the agent (averaged over multiple runs for a given environment). In Fig. \ref{fig:predictions} we show $\chi_\tau,\chi_V$ as a function of the environment $(p_1,p_2)$. We see that $\chi_\tau$ is maximized (yellow colored regions) for larger $T$ in increasingly reward-scarce environments (i.e. when $p_1,p_2\leq 0.5$). While $\chi_V$ shows the opposite trend. $\chi_V$ is maximized in reward-abundant environments (i.e. when $p_1+p_2\geq 1$), while the $p_i$ values are sufficiently distinct. The latter is true for the simple reason that, when $p_1\approx p_2$, the rewards obtained are independent of the behavior.

For suitable assumptions about the distribution of $c$ in the test population, this might suggest observing a greater variance in $\tau,V$ for environments chosen in the yellow regions, as compared to blue regions respectively. Alternatively, as aforementioned, these effects might be observable via an explicit control on $c$  through direct experimental manipulation, for instance, by imposing time constraints or loading working memory.

\section{Conclusions} \label{sec:conclusions}

We proposed a meta-BAMDP framework, extending the scope of metareasoning to also include unknown transition/reward dynamics, thereby providing a normative theory for solving such problems under computational constraints. We present a theoretical formulation of the framework in the Bernoulli Bandit task, and novel theorems to improve computational feasibility. Moreover, we show that theory accounts for human behavioral adaptation to computational constraints in bandit tasks. To our knowledge, this is the first normative explanation of these observations. Finally, our model also provides testable predictions for human behavior in bandit tasks. While additional work is needed to validate model predictions, as well as expanding theoretical understanding and practical implementation of a broader class of meta-BAMDP problems, this work represents theoretical and algorithmic advancement in reinforcement learning in humans and machine learning.



\begin{contributions} 
    P.G. developed the theory, designed the study and wrote the paper.
    T.D.A. wrote the entire code.
\end{contributions}

\begin{acknowledgements} 
    P.G. acknowledges the support of the Walter Benjamin program of the Deutsche Forschungsgemeinschaft (DFG) - project number GO 4136/1-1, which funded part of this project. The authors acknowledge the support of Angela Yu for providing computational and financial resources to be able to execute the project. The authors finally acknowledge Rahul Jain, Johann Bauer and Fred Callaway for helpful discussions during the initial conception of the idea.
\end{acknowledgements}

\bibliography{biblio_2}

@article{RUSSELL1991,
title = {Principles of metareasoning},
journal = {Artificial Intelligence},
volume = {49},
number = {1},
pages = {361-395},
year = {1991},
issn = {0004-3702},
doi = {https://doi.org/10.1016/0004-3702(91)90015-C},
url = {https://www.sciencedirect.com/science/article/pii/000437029190015C},
author = {Stuart Russell and Eric Wefald}
}

@article{jain2023computational,
  title={A computational process-tracing method for measuring people’s planning strategies and how they change over time},
  author={Jain, Yash Raj and Callaway, Frederick and Griffiths, Thomas L and Dayan, Peter and He, Ruiqi and Krueger, Paul M and Lieder, Falk},
  journal={Behavior Research Methods},
  volume={55},
  number={4},
  pages={2037--2079},
  year={2023},
  publisher={Springer}
}

@article{SUTTON1999,
title = {Between MDPs and semi-MDPs: A framework for temporal abstraction in reinforcement learning},
journal = {Artificial Intelligence},
volume = {112},
number = {1},
pages = {181-211},
year = {1999},
issn = {0004-3702},
doi = {https://doi.org/10.1016/S0004-3702(99)00052-1},
url = {https://www.sciencedirect.com/science/article/pii/S0004370299000521},
author = {Richard S. Sutton and Doina Precup and Satinder Singh}
}

@article{hacking1967slightly,
  title={Slightly more realistic personal probability},
  author={Hacking, Ian},
  journal={Philosophy of Science},
  volume={34},
  number={4},
  pages={311--325},
  year={1967},
  publisher={Cambridge University Press}
}

@article{lieder2017strategy,
  title={Strategy selection as rational metareasoning.},
  author={Lieder, Falk and Griffiths, Thomas L},
  journal={Psychological review},
  volume={124},
  number={6},
  pages={762},
  year={2017},
  publisher={American Psychological Association}
}

@article{otto2014physiological,
  title={Physiological and behavioral signatures of reflective exploratory choice},
  author={Otto, A Ross and Knox, W Bradley and Markman, Arthur B and Love, Bradley C},
  journal={Cognitive, Affective, \& Behavioral Neuroscience},
  volume={14},
  pages={1167--1183},
  year={2014},
  publisher={Springer}
}

@article{cogliati2019should,
  title={Should we control? The interplay between cognitive control and information integration in the resolution of the exploration-exploitation dilemma.},
  author={Cogliati Dezza, Irene and Cleeremans, Axel and Alexander, William},
  journal={Journal of Experimental Psychology: General},
  volume={148},
  number={6},
  pages={977},
  year={2019},
  publisher={American Psychological Association}
}

@INPROCEEDINGS{Smit2009,
  author={Smit, S.K. and Eiben, A.E.},
  booktitle={2009 IEEE Congress on Evolutionary Computation}, 
  title={Comparing parameter tuning methods for evolutionary algorithms}, 
  year={2009},
  volume={},
  number={},
  pages={399-406},
  doi={10.1109/CEC.2009.4982974}}

@article{mercer1978,
  title={Adaptive search using a reproductive meta-plan},
  author={Mercer, Robert E and Sampson, JR},
  journal={Kybernetes},
  volume={7},
  number={3},
  pages={215--228},
  year={1978},
  publisher={MCB UP Ltd}
}

@article{huang2019,
  title={A survey of automatic parameter tuning methods for metaheuristics},
  author={Huang, Changwu and Li, Yuanxiang and Yao, Xin},
  journal={IEEE transactions on evolutionary computation},
  volume={24},
  number={2},
  pages={201--216},
  year={2019},
  publisher={IEEE}
}

@article{lieder2018rational,
  title={Rational metareasoning and the plasticity of cognitive control},
  author={Lieder, Falk and Shenhav, Amitai and Musslick, Sebastian and Griffiths, Thomas L},
  journal={PLoS computational biology},
  volume={14},
  number={4},
  pages={e1006043},
  year={2018},
  publisher={Public Library of Science San Francisco, CA USA}
}

@article{callaway2022rational,
  title={Rational use of cognitive resources in human planning},
  author={Callaway, Frederick and van Opheusden, Bas and Gul, Sayan and Das, Priyam and Krueger, Paul M and Griffiths, Thomas L and Lieder, Falk},
  journal={Nature Human Behaviour},
  volume={6},
  number={8},
  pages={1112--1125},
  year={2022},
  publisher={Nature Publishing Group UK London}
}

@article{lieder2014algorithm,
  title={Algorithm selection by rational metareasoning as a model of human strategy selection},
  author={Lieder, Falk and Plunkett, Dillon and Hamrick, Jessica B and Russell, Stuart J and Hay, Nicholas and Griffiths, Tom},
  journal={Advances in neural information processing systems},
  volume={27},
  year={2014}
}

@article{hay2014selecting,
  title={Selecting computations: Theory and applications},
  author={Hay, Nicholas and Russell, Stuart and Tolpin, David and Shimony, Solomon Eyal},
  journal={arXiv preprint arXiv:1408.2048},
  year={2014}
}

@article{sutton1991dyna,
  title={Dyna, an integrated architecture for learning, planning, and reacting},
  author={Sutton, Richard S},
  journal={ACM Sigart Bulletin},
  volume={2},
  number={4},
  pages={160--163},
  year={1991},
  publisher={ACM New York, NY, USA}
}

@book{sutton2018reinforcement,
  title={Reinforcement learning: An introduction},
  author={Sutton, Richard S and Barto, Andrew G},
  year={2018},
  publisher={MIT press}
}

@article{huys2012bonsai,
  title={Bonsai trees in your head: how the pavlovian system sculpts goal-directed choices by pruning decision trees},
  author={Huys, Quentin JM and Eshel, Neir and O'Nions, Elizabeth and Sheridan, Luke and Dayan, Peter and Roiser, Jonathan P},
  journal={PLoS computational biology},
  volume={8},
  number={3},
  pages={e1002410},
  year={2012},
  publisher={Public Library of Science San Francisco, USA}
}

@article{STEYVERS2009,
title = {A Bayesian analysis of human decision-making on bandit problems},
journal = {Journal of Mathematical Psychology},
volume = {53},
number = {3},
pages = {168-179},
year = {2009},
note = {Special Issue: Dynamic Decision Making},
issn = {0022-2496},
doi = {https://doi.org/10.1016/j.jmp.2008.11.002},
url = {https://www.sciencedirect.com/science/article/pii/S0022249608001090},
author = {Mark Steyvers and Michael D. Lee and Eric-Jan Wagenmakers}
}

@inproceedings{D2021,
  title={Meta dynamic programming},
  author={D’Oro, Pierluca and Bacon, Pierre-Luc},
  booktitle={NeurIPS Workshop on Metacognition in the Age of AI: Challenges and Opportunities},
  year={2021}
}

@book{duff2002,
  title={Optimal Learning: Computational procedures for Bayes-adaptive Markov decision processes},
  author={Duff, Michael O'Gordon},
  year={2002},
  publisher={University of Massachusetts Amherst}
}

@article{lin2015,
  title={Metareasoning for planning under uncertainty},
  author={Lin, Christopher H and Kolobov, Andrey and Kamar, Ece and Horvitz, Eric},
  journal={arXiv preprint arXiv:1505.00399},
  year={2015}
}

@article{frazier2008,
  title={A knowledge-gradient policy for sequential information collection},
  author={Frazier, Peter I and Powell, Warren B and Dayanik, Savas},
  journal={SIAM Journal on Control and Optimization},
  volume={47},
  number={5},
  pages={2410--2439},
  year={2008},
  publisher={SIAM}
}

@article{BROWN2022,
title = {Humans adaptively resolve the explore-exploit dilemma under cognitive constraints: Evidence from a multi-armed bandit task},
journal = {Cognition},
volume = {229},
pages = {105233},
year = {2022},
issn = {0010-0277},
doi = {https://doi.org/10.1016/j.cognition.2022.105233},
url = {https://www.sciencedirect.com/science/article/pii/S0010027722002219},
author = {Vanessa M. Brown and Michael N. Hallquist and Michael J. Frank and Alexandre Y. Dombrovski}
}

@article{wilson2014,
  title={Humans use directed and random exploration to solve the explore--exploit dilemma.},
  author={Wilson, Robert C and Geana, Andra and White, John M and Ludvig, Elliot A and Cohen, Jonathan D},
  journal={Journal of experimental psychology: General},
  volume={143},
  number={6},
  pages={2074},
  year={2014},
  publisher={American Psychological Association}
}

@article{wu2022,
  title={Time pressure changes how people explore and respond to uncertainty},
  author={Wu, Charley M and Schulz, Eric and Pleskac, Timothy J and Speekenbrink, Maarten},
  journal={Scientific reports},
  volume={12},
  number={1},
  pages={4122},
  year={2022},
  publisher={Nature Publishing Group UK London}
}

@article{zhang2013,
  title={Forgetful Bayes and myopic planning: Human learning and decision-making in a bandit setting},
  author={Zhang, Shunan and Yu, Angela J},
  journal={Advances in neural information processing systems},
  volume={26},
  year={2013}
}

\newpage

\onecolumn

\title{A resource rational account of the explore-exploit dilemma via metareasoning \\(Supplementary Material)}
\maketitle

\appendix
\section{Appendix} \label{app:supplementary}

\subsection{Markov Decision Process - MDP} \label{app:MDP_def}

A Markov decision process (MDP) is formally defined by a tuple $(\mathcal{S}, \mathcal{A}, \mathcal{P}, \mathcal{R},T)$. Here $\mathcal{S}$ is the set of states of the environment, $\mathcal{A}$ the set of actions available to the agent, $\mathcal{P}: \mathcal{S} \times \mathcal{A} \times \mathcal{S} \rightarrow [0, 1]$ the transition probability function and  $\mathcal{R}: \mathcal{S} \times \mathcal{A} \times  \mathbb{R} \rightarrow [0,1]$  the reward distribution. The goal in an MDP is to find an optimal policy \(\pi^*(s,t)\) (one that maximizes the expected cumulative reward from any given state and time $t$ to a terminal horizon $T$). Formally, the objective is to maximize:

\begin{equation}\begin{split}
   V^{\pi}(s,t) &=  \sum_a \pi(a|s,t) \Big[\int r\mathcal{R}(r|s,a)dr  + \sum_{s'} \mathcal{P}(s'|s,a) V^{\pi}(s',t+1)\Big] , 
\end{split}
\end{equation}

with the terminal condition $V^\pi(s,T) = 0$. \( \pi(a|s,t) \) specifies the conditional probability of taking an action $a$ when the state is \( s \) and time is $t$, and \( V^{\pi}(s,t) \) is the value function under policy \(\pi\). The optimal policy \(\pi^*\) is then defined as $\pi^*(a|s,t) = \arg\max_\pi V^\pi(s,t)$. If the problem an agent faces can be modelled as an MDP, the agent would need access to $\mathcal{P}$ and $\mathcal{R}$ in order to find $\pi^*$. This process (of finding the optimal policy) is usually referred to as \textit{planning} (as opposed to \textit{learning}, which refers to learning the transition distributions from experience). In a more general setting, both $\mathcal{P}$ and $\mathcal{R}$ may be initially unknown to the agent,  and it would also need to learn them from experience. The agent could, for instance, learn the transition and reward distributions via Bayesian inference. If so, we end up with a BAMDP.

\subsection{Bayes-Adaptive Markov Decision Process - BAMDP}\label{app:BAMDP_def}

A Bayes-adaptive Markov decision process (BAMDP) extends the standard MDP framework by incorporating uncertainty about the reward and transition distributions, defined by the tuple \((\mathcal{S}_B, \mathcal{A}, \mathcal{P}_B, \mathcal{R}_B, b_0,T)\). Here, \(\mathcal{S}_B = \mathcal{S} \times \mathcal{B}\) represents the augmented state space, where \(\mathcal{S}\) is the physical state space, and \(\mathcal{B}\) is the belief space encapsulating probabilistic beliefs over parameterized distributions \(\mathcal{P}\) and \(\mathcal{R}\) (parameters \(\theta\), with belief \(b: \Theta \to [0,1]\)). The action space \(\mathcal{A}\), transition model \(\mathcal{P}_B: \mathcal{S}_B \times \mathcal{A} \times \mathcal{S}_B \rightarrow [0, 1]\), and reward distribution \(\mathcal{R}_B: \mathcal{S}_B \times \mathcal{A} \times \mathbb{R} \rightarrow [0,1]\).

The initial belief distribution over the models, \(b_0\), sets the starting conditions. The goal of a BAMDP is to derive an optimal policy \(\pi^*(s,b,t)\), one that maximizes the expected cumulative reward, accounting for model uncertainty. Formally, the maximization target is given by:

\begin{equation}
\begin{split}
    V^{\pi}(s,b,t) &=  \sum_a \pi(a|s,b,t) \Big[\int r\mathcal{R}_B(r|s,b,a)dr  +  \sum_{s',b'} \mathcal{P}_B(s',b'|a,s,b) V^{\pi}(s',b',t+1)\Big] ,
\end{split}
\end{equation} with the terminal condition $V^\pi (s,b,T) = 0$. The optimal policy \(\pi^* = \arg\max_\pi V^\pi(s, b,t)\) is defined as the one that maximizes \(V^\pi\) for all \((s, b) \in \mathcal{S}_B\) and $t\in\{0,\cdots,T\}$.

There are some crucial things to take note of here. A BAMDP is structurally distinct from a typical model-based RL algorithm like Dyna \citep{sutton1991dyna}. Not only is the agent updating its beliefs about $(\mathcal{P},\mathcal{R})$ to then solve the implied MDP, but also, $\mathcal{P}_B$ incorporates the evolution of the beliefs themselves. Therefore, while the agent might not know the environment dynamics, it could still make use of its belief update dynamics to guess the future status of its knowledge. Therefore, when moving from an MDP to a BAMDP, we "loosen" the restriction on the part of the agent -- i.e. from requiring it to know $(\mathcal{P},\mathcal{R})$, to requiring it to know $\mathcal{P}_B$. This apparent generality doesn't come for free: the state space of BAMDP is much larger than that of the underlying MDP, since an arbitrary probability distribution over continuous r.v.'s is infinite-dimensional.

\subsection{Pseudocode}\label{app:pseudocode}
The complete algorithm can be accessed via the url : \texttt{https://github.com/Dies-Das/meta-BAMPD-data}. The algorithm to find the solution to the meta-BAMDP involves two routines. First, to generate a pruned meta-graph and second to perform backward induction on this meta-graph, to find the optimal meta-policy. The latter is rather straight-forward, but the former is slightly complex to present and therefore for ease of understanding we provide a pseudocode below (Algorithms \ref{alg:meta_graph} and \ref{alg:search_trajectories}).

\begin{algorithm}
	\caption{Constructing the Pruned Meta-Graph}
	\label{alg:meta_graph}
	\begin{algorithmic}[1]
		\State Initialize the meta-graph $G$ with a root node containing $(\bm{b}_0, \tilde{b}_0)$.
		\State Initialize a queue $W$ and enqueue the root node.
		\While{$W$ is not empty}
		\State Dequeue a node $(\bm{b}, \tilde{b})$ from $W$.
		\State Determine the terminal action $a_\perp$ for $(\bm{b}, \tilde{b})$ using Eq.~\ref{eq:tree_to_action}.
		\For{each subsequent belief state $\bm{b}'$ for action $a_\perp$}
		\State Update $\tilde{b}'$ to be the subgraph of $\tilde{b}$ starting from $\bm{b}'$, i.e.\ $\tilde b' = R_{\bm b'}(\tilde b)$.
		\State Create new node $(\bm{b}', \tilde{b}')$ in $G$ if not already present.
		\State Add an edge from $(\bm{b}, \tilde{b})$ to $(\bm{b}', \tilde{b}')$.
		\State Enqueue $(\bm{b}', \tilde{b}')$ into $W$.
		\If{$Q(i,\bm b\mid \tilde b) \le Q^*(j,\bm b\mid \tilde b),\ \forall j\neq i$ \Comment{Cor.~\ref{Thm:Termination}}}
		\State \textbf{Call:} \textsc{SearchComputationalTrajectories}$((\bm{b}, \tilde{b}), G, W)$.
		\EndIf
		\EndFor
		\EndWhile
		\State \Return the pruned meta-graph $G$.
	\end{algorithmic}
\end{algorithm}

\begin{algorithm}
	\caption{\textsc{SearchComputationalTrajectories}}
	\label{alg:search_trajectories}
	\begin{algorithmic}[1]
		\State \textbf{Input:} Current node $(\bm{b}, \tilde{b})$, meta-graph $G$, queue $W$.
		\State Start depth-first search (DFS) through all possible computational expansion trajectories from $(\bm{b}, \tilde{b})$ ending with a terminal action.
		\State Use the restrictions from Corollaries \ref{Thm:M_state_computaion} and \ref{Thm:Termination}.
		\For{each trajectory in the above set}
		\State Determine the terminal action $a_\perp(\bm{b}, \tilde{b}')$ for the last state $(\bm{b}, \tilde{b}')$ in the trajectory.
		\State Calculate subsequent belief states after executing $a_\perp(\bm{b}, \tilde{b}')$, resulting in $(\bm{b}', \tilde{b}'')$.
		\If{$a_\perp(\bm{b}, \tilde{b}') \ne a_\perp(\bm{b}, \tilde{b})$ \Comment{Thm.~\ref{Thm:MC}}}
		\State Add new nodes $(\bm{b}', \tilde{b}'')$ to $G$ and $W$.
		\EndIf
		\If{$Q(i,\bm b\mid \tilde b') \le Q^*(j,\bm b\mid \tilde b'),\ \forall j\neq i$ \Comment{Cor.~\ref{Thm:Termination}}}
		\State Terminate the search along this trajectory.
		\ElsIf{Any terminal condition $X$ is met (see Sec.~\ref{app:robust_approximations} below)}
		\State Terminate the search along this trajectory.
		\EndIf
		\EndFor
	\end{algorithmic}
\end{algorithm}

\subsection{Robustness of the solution} \label{app:robust_approximations}
In order to test the robustness of the approximate solution, we loosen the restriction of myopic approximation from Sec. 5, in three different ways. These correspond to three distinct approximation schemes (or terminal conditions X in line 12 in Alg. \ref{alg:search_trajectories}) that we tested. First as in the main text, we upper bound the maximum size $|\tilde b|$ of a planning-belief $\tilde b$ that the agent can posses in any state by $k$. This may be viewed as bounding the working memory of an agent. Alternatively, we could also restrict the maximum number of computational actions $k_c$ that the agent is allowed to take between two consecutive terminal actions. Lastly, as the rewards diminish geometrically with depth, we bound the maximal depth $d$ we search for minimal mind-changers. While staying within the bounds of the computational resources at our disposal, we find that the optimal solution remains invariant for $k_c\in\{1,2,3\}$, for $2\leq k \leq 16$, and $d\leq 4$.

\subsection{Comparing heuristic policies to optimal meta-policies} 
\begin{figure}
\centering
\subfloat[]{\includegraphics[width = 1.6in]{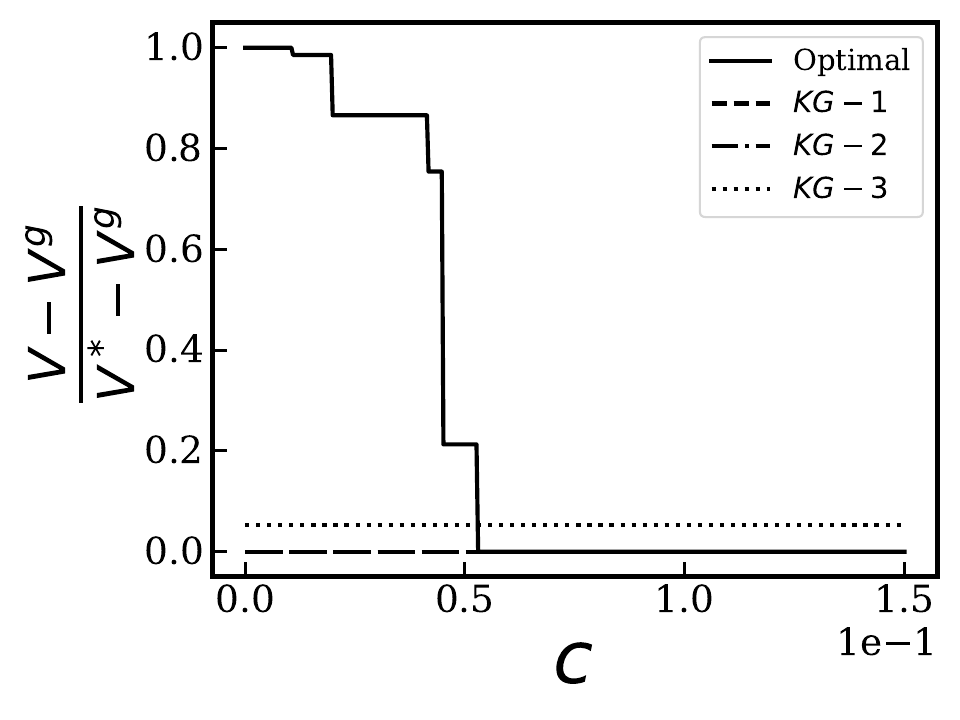}}
\subfloat[]{\includegraphics[width = 1.6in]{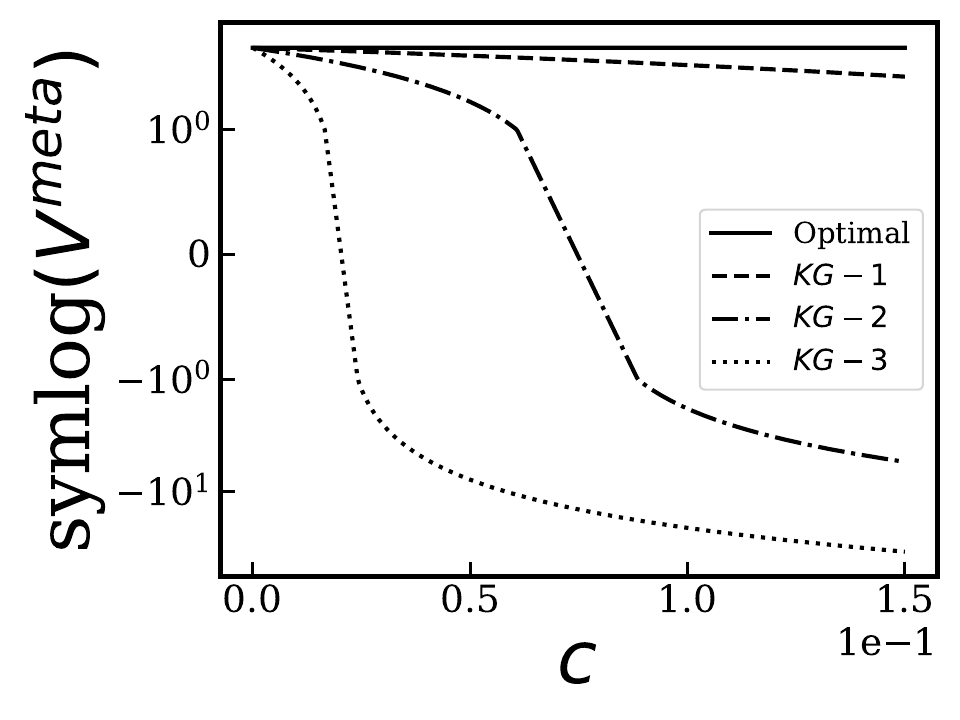}}
\subfloat[]{\includegraphics[width = 1.6in]{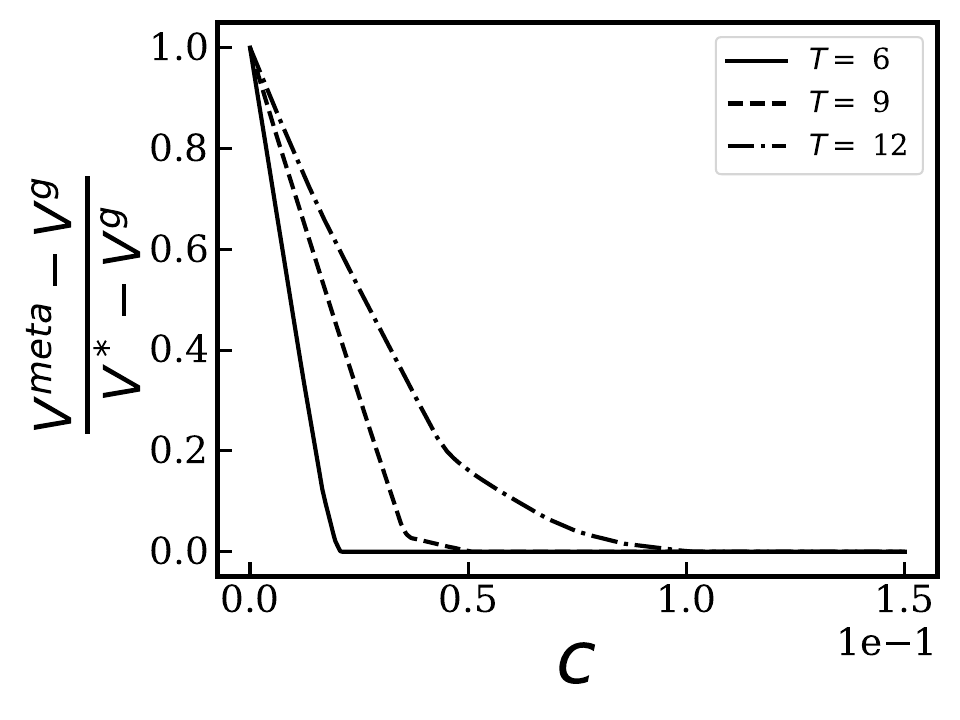}}

\caption{(a) Normalized reward gained as a function of computational cost for $(N,T)=(2,9)$, for the meta-optimal policy and the knowledge gradient KG-$l$ polices with $l \in (1,2,3)$. (b) The meta-value (on a symlog scale) for $(N,T)=(2,9)$, for the meta-optimal policy and the knowledge gradient KG-$l$ polices with $l \in (1,2,3)$. (c) Normalized meta-value for optimal policies as a function of $T\in \{6,9,12\}$ and computational costs.}
\label{fig:KG_compare}
\end{figure}
\subsubsection{Knowledge gradient policies} \label{app:KG_policies}
In order to provide a point of reference for understanding the behavior of the meta-optimal policies, we compare them to KG-$l$ policies\footnote{We consider KG policies as they are compatible with the metareasoning framework. Other heuristic policies such as Thompson Sampling (TS) or Upper Confidence Bound (UCB) do not admit a compatible notion of a computation that we consider here. Therefore we cannot comment on their meta-optimality and consequently they do not make a meaningful comparison point.} (see Fig. \ref{fig:KG_compare}). These policies are $l$-step look-ahead policies, where the look-ahead is computed at each time step. In the context of the meta-BAMDP framework these policies, at each time step construct $\tilde b$ of depth $l$ starting from the root node $\bm b$ at that time-step and take the optimal action given by $a_\perp$ (Eq. \ref{eq:tree_to_action}). At the next time-step, the previous $\tilde b$ is forgotten and a new $\tilde b^\prime$ of depth $l$ is constructed again from the new root node $\bm b'$ obtained via a 1-step transition from $\bm b$. The original KG policy of \citep{frazier2008} would be a KG-$1$ policy. In Fig. \ref{fig:KG_compare} (a), we compare the normalized performance of the KG-$l$ with that of the meta-optimal policy for $(N,T) = (2,9)$. We find that both KG policies for $l= 1,2$ seem to perform equivalently to the greedy policy (or KG-$0$) and KG-$3$ performs only marginally better than the two.

Notice that these policies are rather wasteful in the computations they perform as the number of computations grows polynomially in $l$ and exponentially in $N$ (see \ref{app:complexity_concerns}). This fact can be seen explicitly in Fig. \ref{fig:KG_compare} (b) where we present the unnormalized (symmetric logarithm of the) metalevel value $V^{\text{meta}} = V - \langle n_c \rangle c$ for the optimal and the KG-$l$ policies. Here $\langle n_c \rangle$ represents the ensemble average number of computations done by the policy. One may view $V^{\text{meta}}$ as the metalevel performance measure, as it takes into account both the task performance and the computational costs incurred. As can be clearly seen, the KG policies have a rather low metalevel performance as compared to the optimal policies. In Fig. \ref{fig:KG_compare} (c) we also demonstrate the metalevel performance of the optimal policies as a function of $T$. 

Finally, we'd like to remind the reader that our goal is not to create policies that outperform known heuristic policies. This can be seen by the fact that we are making use of $Q^*$ in order to obtain the meta-optimal policies (see Corollary \ref{Thm:Termination} for instance). Instead our goal is to obtain the behavior generated by these meta-optimal policies to compare with and explain human behavior. Nonetheless, insights from this exercise are likely to also help find better heuristic policies. For instance, if one obtains an approximation of $\mathcal{M}$ states, one might be able to construct a hybrid KG algorithm which is KG-$0$ in $\mathcal{M}$ states and KG-$l$ otherwise, which will certainly outperform the vanilla KG-$l$ policy.

\subsubsection{A generalized linear model including uncertainty bonus}\label{app:heuristics_fit_procedure}
For each computational cost $c$ which we choose from 400 uniformly distributed points in the interval $[0,0.15]$, we solve the meta-BAMDP problem and find the optimal meta-policy. With each such meta-policy we run $10^5$ simulations of agents with the initial condition $(\bm b,\tilde b) = (\bm 0, \tilde 0)$ as mentioned in the main text. The rewards are sampled from a symmetric environment $p_1=p_2=0.5$ which is unknown to the agent. For each simulation run, we perform a fit to the obtained action and reward trajectory, by minimizing the negative log likelihood under the heuristic policy $\pi(a=i|\bm b) \propto e^{-(\beta \hat \mu_i + \omega \hat \sigma_i)}$, as a function of $\omega$. We consider the $\beta \times \omega$ space to be bounded by the square $(0,100)\times (-10,10)$, to remain broadly consistent with the work of \citep{BROWN2022}.  As the optimal policies from the BAMDP are deterministic (except when ties are broken), the estimated value $\beta$ turns out to be suitably large ($\approx 80$ in our case). After obtaining the best fit $\omega$ for each simulation run, we consider its average over all the simulation runs.

\subsection{Proofs of the theorems}\label{app:proofs}

\subsubsection{Proof of Thm. \ref{Thm:MC}}

\begin{proof}
We prove the theorem by contradiction. Assume $\tilde b^\prime \neq \tilde b$ and $a_\perp(\bm b, \tilde b) = a_\perp(\bm b, \tilde b^\prime)$. We will show that this assumption leads to a contradiction of the optimality of $\pi^*$. We begin by noting two facts. \begin{itemize}
    \item Since $\tilde b^\prime$ includes all nodes and edges of $\tilde b$, we have $\tilde b \subset \tilde b^\prime$.
    \item The subjective $Q$-value $\mathcal{K}(\tilde b)(\bm b,a)$ depends only on the portion of $\tilde b$ that is \textit{reachable} from $\bm b$, denoted $R_{\bm b}(\tilde b)$:
    \[
    \mathcal{K}(\tilde b)(\bm b,a) = \mathcal{K}(R_{\bm b}(\tilde b))(\bm b,a).
    \]
\end{itemize} As stated in the theorem we assume the current state to be $(\bm b, \tilde b)$ and now construct two meta-policies as follows:
\begin{itemize}
    \item \(\pi_1 := \pi^*\): This policy prescribes computing from the starting state $(\bm b, \tilde b)$ until $(\bm b,\tilde b^\prime)$ and then terminating, i.e. in state $(\bm b, \tilde b)$ and in subsequent states only chooses $a\in \mathcal{C}$ until reaching $(\bm b,\tilde b^\prime)$, then chooses $\perp$.
    \item \(\pi_2\): This policy imitates $\pi_1$ in all states except in the following states. It chooses $\perp$ immediately in $(\bm b, \tilde b)$. In state(s) \((\bm b^\prime, R_{\bm b^\prime}(\tilde b))\), chooses $a\in \mathcal{C}$ until reaching $R_{\bm b^\prime}(\tilde {b}^\prime))$, then chooses $\perp$. Here $\bm b^\prime$ is any belief reachable from $\bm b$ in one physical action.

\end{itemize} We now consider the behavior generated by the two trajectories. Since we assume \(a_\perp(\bm b, \tilde b) = a_\perp(\bm b, \tilde b^\prime)\), both policies lead to the same terminal action from the state \((\bm b, \tilde b)\). After termination, consider the resulting belief state \(\bm b^\prime\). The subsequent states reached by the two policies are:
\begin{itemize}
    \item \(\pi_1\): \((\bm b^\prime, R_{\bm b^\prime}(\tilde b^\prime))\),
    \item \(\pi_2\): \((\bm b^\prime, R_{\bm b^\prime}(\tilde b))\).
\end{itemize}

Since $\tilde b \subset \tilde b^\prime$, it follows that:
\[
R_{\bm b^\prime}(\tilde b) \subseteq R_{\bm b^\prime}(\tilde b^\prime).
\]

From the state \((\bm b^\prime, R_{\bm b^\prime}(\tilde b))\), the policy \(\pi_2\) prescribes computation till the agent reaches the state \((\bm b^\prime, R_{\bm b^\prime}(\tilde b^\prime))\) and then it terminates. In all other states, it imitates $\pi_1$. Thus, the two policies generate the same physical behavior. The only difference between \(\pi_1\) and \(\pi_2\) lies in the timing of the computations:
\begin{itemize}
    \item \(\pi_1\) performs computations when the belief state is $\bm b$ and then terminates. 
    \item \(\pi_2\) terminates immediately and performs computations after reaching the belief state \(\bm b^\prime\).
\end{itemize}

Since the computations needed to transition from \(R_{\bm b^\prime}(\tilde b)\) to \(R_{\bm b^\prime}(\tilde b^\prime)\) are a subset of the computations performed to expand \(\tilde b\) to \(\tilde b^\prime\), i.e. we have:
\[
|\tilde b^\prime| - |\tilde b| \geq |R_{\bm b^\prime}(\tilde b^\prime)| - |R_{\bm b^\prime}(\tilde b)|.
\]

This inequality (see Lemma \ref{lem:ineq}) shows that \(\pi_2\) incurs a smaller (or at least equal) computational cost compared to \(\pi_1\).

We have shown that \(\pi_2\) generates the same physical behavior as \(\pi_1\) but incurs smaller or equal computational cost. I.e. \(\pi_2\) can only be better than \(\pi_1\). This contradicts the assumption that \(\pi_1 = \pi^*\) is the optimal meta-policy. Hence, for \(\pi^*\) to be optimal, it must satisfy the condition that either:
\begin{enumerate}
    \item $\tilde b^\prime = \tilde b$, or
    \item $a_\perp(\bm b, \tilde b) \neq a_\perp(\bm b, \tilde b^\prime)$ when $\tilde b^\prime \neq \tilde b$.
\end{enumerate}
Thus, the theorem is proved.
\end{proof}

One interesting observation on Theorem \ref{Thm:MC} is as follows. Assume that in the case of two arms (i.e. $N=2$), the optimal meta-policy either prescribes terminating immediately, thus pulling some arm A or prescribes computation until the agent's mind-changes. But if the policy doesn't terminate immediately, then the agent already knows that the optimal action in such a scenario is B. This way, one might think, the agent can cheat its way out of computing. However, it is important to know that computation has downstream effects as the computational beliefs $\tilde b$ are carried forward to the next time-step, thereby impacting which future states are visited. Therefore, such cheating could result to sub-optimal futures. Secondly, and more technically, the agents in our framework are restricted to act according to Eq. \ref{eq:greedy_act}, and can therefore not cheat in this manner.

\subsubsection{Proof of Corollary \ref{Thm:MMC}}

\begin{proof}
Let $\pi^*$ be the optimal meta-policy, and suppose that in state $(\bm b, \tilde b)$, $\pi^*$ prescribes computation until $(\bm b, \tilde b^\prime)$, where $\tilde b^\prime \neq \tilde b$. Consider the computation trajectory from $\tilde b$ to $\tilde b^\prime$ as:
\[
\tilde b \to \tilde b_1 \to \tilde b_2 \to \cdots \to \tilde b_k \to \tilde b^\prime,
\]
where $\tilde b \subset \tilde b_i \subset \tilde b_{i+1}\subset \tilde b^\prime$,  $\forall i$. More specifically, for the Bernoulli bandit scenario $|\tilde b_{i+1} | -|\tilde b_{i}| = N $ (see Fig. 1).

From Theorem \ref{Thm:MC}, we know that:
\[
a_\perp(\bm b, \tilde b) \neq a_\perp(\bm b, \tilde b^\prime).
\]

Now, suppose for contradiction that there exists some $i \leq k$ such that:
\[
a_\perp(\bm b, \tilde b_i) = a_\perp(\bm b, \tilde b^\prime).
\]

By applying Theorem \ref{Thm:MC} from the starting state $(\bm b, \tilde b_i)$, we must have:
\[
\tilde b_i = \tilde b^\prime.
\]

However, this contradicts the assumption that $\tilde b_i \subset \tilde b^\prime$ and $\tilde b_i \neq \tilde b^\prime$ for all $i \leq k$. Therefore, such an $i$ cannot exist, and for every intermediate subgraph $\tilde b_i$ along the trajectory, we have:
\[
a_\perp(\bm b, \tilde b_i) \neq a_\perp(\bm b, \tilde b^\prime).
\]

This ensures that $\pi^*$ only prescribes computation until reaching the smallest $\tilde b^\prime$ (along any given path) where the terminal action changes, making $\pi^*$ a minimal mind-changer.

\end{proof}

\subsubsection{Proof of Thm. \ref{Thm:monotonicity}}

\begin{proof}
    Recall the belief state:
\[
\bm{b} = (\alpha_1,\beta_1,\alpha_2,\beta_2,\ldots,\alpha_N,\beta_N),
\]
where each pair $(\alpha_i,\beta_i)$ parameterizes the posterior Beta distribution for the reward probability of arm $i$. Let $\tilde{b}$ be the current planning-belief subgraph consisting only of the root node $\bm{b}$.

Define $p_i$ as the (subjective) expected reward probability of arm $i$ given the belief $\bm{b}$:
\[
p_i = \frac{\alpha_i + 1}{\alpha_i + \beta_i + 2}.
\]
Let $\tau$ be the remaining number of actions.

We now define the subjective value function under $\tilde{b}$ as:
\[
V(\bm{b}|\tilde{b}) = \max_a\mathcal{K}(\tilde b)(\bm b, a) = \max_a Q(a,\bm{b}|\tilde{b}),
\] where $Q(a,\bm{b}|\tilde{b})$ represents the subjective $Q$-function of the meta-BAMDP under the computational belief $\tilde b$. Initially, since no computations have been done and the only estimate we have is the terminal heuristic from Eq. 7, we have:
\[
Q(a=i,\bm{b}|\tilde{b}) = p_i \tau \quad \text{for all } i \in \{1,\dots,N\}.
\]

Without loss of generality, we consider performing a single computational action on arm $i$. After this computation, we obtain a refined planning-belief $\tilde{b}_i$. Note that $Q(a=j,\bm{b}|\tilde{b}_i) = Q(a=j,\bm{b}|\tilde{b})$ for all $j \neq i$, since the computation only adds detail to the reachable subgraph rooted at arm $i$.

When we expand arm $a=i$, we introduce two possible outcomes:
- W (Win): The reward is obtained from arm $i$, effectively increasing $\alpha_i$ by 1.
- L (Loss): No reward from arm $i$, effectively increasing $\beta_i$ by 1.

This leads to updated probabilities for these child states:
\[
\hat{p}_i^\text{W} = \frac{\alpha_i + 2}{\alpha_i + \beta_i + 3}, \quad
\hat{p}_i^\text{L} = \frac{\alpha_i + 1}{\alpha_i + \beta_i + 3}.
\]

After the expansion, the value of the child states $V^\text{W}$ and $V^\text{L}$ must consider all arms. Let:
\[
q = \max_{j \neq i} p_j.
\]
Then:
\[
V^\text{W} = \max(\hat{p}_i^\text{W}, q)(\tau-1), \quad V^\text{L} = \max(\hat{p}_i^\text{L}, q)(\tau-1).
\]

The $Q$-value for arm $i$ after the expansion is:
\[
Q(a=i,\bm{b}|\tilde{b}_i) = p_i(1 + V^\text{W}) + (1 - p_i) V^\text{L}.
\]

Our goal is to compare $Q(a=i,\bm{b}|\tilde{b}_i)$ to the original $Q(a=i,\bm{b}|\tilde{b}) = p_i \tau$. As in the two-armed case, consider three main orderings of $\hat{p}_i^\text{W}, \hat{p}_i^\text{L}$ relative to $q$:

\begin{enumerate}
    \item \textbf{Case 1:} $\hat{p}_i^\text{W} > \hat{p}_i^\text{L} > q$.
    
    In this case:
    \[
    V^\text{W} = \hat{p}_i^\text{W}(\tau-1), \quad V^\text{L} = \hat{p}_i^\text{L}(\tau-1).
    \]
    Thus:
    \[
    Q(a=i,\bm{b}|\tilde{b}_i) = p_i [1 + \hat{p}_i^\text{W}(\tau-1)] + (1 - p_i)\hat{p}_i^\text{L}(\tau-1).
    \]

    Substituting the definitions of $p_i$, $\hat{p}_i^\text{W}$, and $\hat{p}_i^\text{L}$, one finds that in this scenario:
    \[
    Q(a=i,\bm{b}|\tilde{b}_i) = p_i \tau = Q(a=i,\bm{b}|\tilde{b}).
    \]
    Thus, in the scenario where both child probabilities exceed $q$, the single computation does not change the $Q$-value from its initial value.

    \item \textbf{Case 2:} $\hat{p}_i^\text{W} > q > \hat{p}_i^\text{L}$.
    
    In this scenario:
    \[
    V^\text{W} = \hat{p}_i^\text{W}(\tau-1), \quad V^\text{L} = q(\tau-1).
    \]
    Hence:
    \[
    Q(a=i,\bm{b}|\tilde{b}_i) = p_i[1 + \hat{p}_i^\text{W}(\tau-1)] + (1 - p_i)q(\tau-1).
    \]

    Since $q > \hat{p}_i^\text{L}$, replacing $\hat{p}_i^\text{L}(\tau-1)$ with $q(\tau-1)$ in the second term strictly increases the $Q$-value relative to Case 1. Therefore:
    \[
    Q(a=i,\bm{b}|\tilde{b}_i) > p_i \tau = Q(a=i,\bm{b}|\tilde{b}).
    \]

    \item \textbf{Case 3:} $q > \hat{p}_i^\text{W} > \hat{p}_i^\text{L}$.
    
    Here:
    \[
    V^\text{W} = q(\tau-1), \quad V^\text{L} = q(\tau-1).
    \]
    Thus:
    \[
    Q(a=i,\bm{b}|\tilde{b}_i) = p_i[1 + q(\tau-1)] + (1 - p_i)q(\tau-1) = p_i + q(\tau-1).
    \]

    This is even larger than the value in Case 2, since now both $V^\text{W}$ and $V^\text{L}$ use $q$, which by assumption is larger than $\hat{p}_i^\text{L}$. Hence:
    \[
    Q(a=i,\bm{b}|\tilde{b}_i) > Q(a=i,\bm{b}|\tilde{b}).
    \]
\end{enumerate}

In all cases, after performing a single computational expansion on arm $i$, the $Q$-value for arm $i$ is never reduced. In the worst case (Case 1), it stays the same, and in other cases, it strictly increases. Since the value function $V(\bm b|\tilde b)$ is defined as the maximum of the $Q$-values over all actions, and each computational action either maintains or increases the $Q$-value for the expanded sub-DAG, we have:
\[
V(\bm b|\tilde b) \leq V(\bm b|\tilde b_i).
\] By iterating this argument for additional computations and expansions, for any extended planning-belief $\tilde b' \supseteq \tilde b$, it follows that:
\[
V(\bm b|\tilde b) \leq V(\bm b|\tilde b'),
\]

\end{proof}

\subsubsection{Proof of Corollary \ref{Thm:Termination}}

\begin{proof}
From Thm. \ref{Thm:monotonicity}, the subjective $Q$-value for any action $a$ at a belief $\bm b$ is bounded above by the optimal $Q^*$-value when the entire reachable subgraph from $\bm b$ has been fully expanded. Formally:
\[
Q(i, \bm b | \tilde b) \leq Q^*(i, \bm b),
\]
for any $i$ and computational graph $\tilde b$.

Now, consider a state $(\bm b, \tilde b)$ where for some action $i$ and $\forall j\neq i$ such that:
\[
Q(i, \bm b | \tilde b) \geq Q^*(j, \bm b).
\]
For any computational graph $\tilde b^\prime$ that is a supergraph of $\tilde b$, i.e. $\tilde b^\prime \supseteq \tilde b$, it follows from Thm. \ref{Thm:monotonicity} that:
\[
\arg\max_a Q(a, \bm b | \tilde b^\prime) = i.
\]

From Thm. \ref{Thm:MC} and Corollary \ref{Thm:MMC}, the optimal meta-policy $\pi^*$ is a minimal mind-changer, which means it terminates computations as soon as the terminal action $a_\perp$ becomes invariant to further computations. In this case, since $a_\perp = i$ is invariant to any further expansion of $\tilde b$, the optimal policy $\pi^*$ will immediately terminate in $(\bm b, \tilde b)$.

Thus, $\pi^*$ prescribes termination in all states $(\bm b, \tilde b)$ where $\exists \, i$, such that $ Q(i, \bm b | \tilde b) \geq Q^*(j, \bm b)$, $\forall j\neq i$.
\end{proof}

\subsubsection{Proof of Corollary  \ref{Thm:M_state_computaion}}

\begin{proof}
Let $\bm b \in \mathcal{M}$, and let $\tilde b$ be any computational graph. By the definition of $\mathcal{M}$, for the belief $\bm b$, the subjective $Q$-value for the greedy arm $i$ satisfies:
\[
Q(i, \bm b | \tilde b) \geq Q^*(j, \bm b) \quad \forall j \neq i.
\]

The subjective value $V(\bm b | \tilde b)$ is defined as:
\[
V(\bm b | \tilde b) = \max_a Q(a, \bm b | \tilde b).
\]

For beliefs $\bm b \in \mathcal{M}$ and $\forall \tilde b$, the above implies:
\[
V(\bm b | \tilde b) = Q(a=i, \bm b | \tilde b).
\]

Since the subjective value $V(\bm b | \tilde b)$ depends solely on the $Q$-value of the greedy arm $i$, it is independent of the $Q$-values of the non-greedy arms ($a \neq i$). Therefore, further computation along any non-greedy arm leaves the subjective value $V(\bm b | \tilde b)$ invariant and never leads to mind-changing.  

As a result, the optimal meta-policy $\pi^*$ does not prescribe computation along non-greedy arms for any belief $\bm b \in \mathcal{M}$.

\end{proof}

\begin{lemma} \label{lem:ineq}
Let $\tilde b \subseteq \tilde b^\prime$ be two planning-belief subgraphs, and let $R_{\bm b^\prime}(\tilde b)$ and $R_{\bm b^\prime}(\tilde b^\prime)$ denote the reachable subgraphs from belief state $\bm b^\prime$. Then the following inequality holds:
\[
|\tilde b^\prime| - |\tilde b| \geq |R_{\bm b^\prime}(\tilde b^\prime)| - |R_{\bm b^\prime}(\tilde b)|.
\]
\end{lemma}

\begin{proof}
Since $\tilde b \subseteq \tilde b^\prime$, it follows that $R_{\bm b^\prime}(\tilde b) \subseteq R_{\bm b^\prime}(\tilde b^\prime)$. Consequently, the additional nodes and edges in $R_{\bm b^\prime}(\tilde b^\prime)$ relative to $R_{\bm b^\prime}(\tilde b)$ must also belong to $\tilde b^\prime \setminus \tilde b$. Formally:
\[
R_{\bm b^\prime}(\tilde b^\prime) \setminus R_{\bm b^\prime}(\tilde b) \subseteq \tilde b^\prime \setminus \tilde b.
\]

Taking the cardinalities of the sets yields:
\[
|R_{\bm b^\prime}(\tilde b^\prime) \setminus R_{\bm b^\prime}(\tilde b)| \leq |\tilde b^\prime \setminus \tilde b|.
\]

Rewriting the set differences in terms of their cardinalities, we have:
\[
|R_{\bm b^\prime}(\tilde b^\prime) \setminus R_{\bm b^\prime}(\tilde b)| = |R_{\bm b^\prime}(\tilde b^\prime)| - |R_{\bm b^\prime}(\tilde b)|,
\]
and
\[
|\tilde b^\prime \setminus \tilde b| = |\tilde b^\prime| - |\tilde b|.
\]

Substituting these expressions into the inequality gives:
\[
|R_{\bm b^\prime}(\tilde b^\prime)| - |R_{\bm b^\prime}(\tilde b)| \leq |\tilde b^\prime| - |\tilde b|.
\]
\end{proof}

\subsection{Comments on complexity} \label{app:complexity_concerns}
Performing backward induction on a DAG $G$ is linear in the number of edges. The size of the base graph (number of nodes) for a Bernoulli bandit task with $N$ arms and horizon $T$ is $O(T^{2N})$. Size of the meta-graph is exponential in the size of base graph, while being upper bounded by $O(2^{T^{2N}})$.  This means that the brute-force optimization for the meta-policy is exponential in $T,N$. On our computers, the brute-force solution  for $N=2$ can only be found for $T\leq2$. By making use of the pruning theorems from Sec. 5 (in the main text), we are able to find the optimal solutions for $N=2$ and  $T\leq 6$. 
    
Given that one necessarily needs to solve the problem approximately, we are unable to provide comparisons of our proposed solution (at larger values of $T$) with the true optimal solution. However, we do test the robustness of the approximate solution within some range of increasingly weaker approximations, as mentioned in  \ref{app:robust_approximations}. The presented analysis is only possible because our monotonicity arguments help us prune the meta-graph (which provides an exponential improvement over the brute-force approach, as we avoid computing in $\mathcal{M}$ states, the size of the (pruned) meta-graph with our approach is upper bounded by $O(2^{T^{2N}/c})$ for some constant $c$ proportional to $|\mathcal{M}|/|\mathcal{S}|$).

\subsection{Relation to meta-MDP formulations}\label{app:relation_to_MMDP}
The presented study differs from existing formulations of metareasoning \citep{lin2015,callaway2022rational,hay2014selecting} in the following crucial ways.
\begin{enumerate}
    \item The meta-BAMDP involves two distinct belief spaces - capturing epistemic uncertainty about the environment $\mathcal{B}$ and the \textit{computational uncertainty} arising out of approximate planning $\tilde {\mathcal{B}}$. As the traditional metareasoning frameworks consider complete knowledge of the environment, they do not make any distinctions between the two kinds of uncertainties. Such a distinction between these two forms of uncertainties has been made previously in philosophical literature (for instance see \citep{hacking1967slightly}). 
    \item Another issue, which has been identified previously is that metareasoning has only been adapted for single-shot decision-making (the single shot might even be a policy). I.e. that all metareasoning happens prior to the execution of behavior. The meta-BAMDP framework allows us to study scenarios where thinking and acting are \textit{interleaved}.
    \item A further distinction from previous efforts has been, that in our formulation of the meta-BAMDP, we do not apriori know the distribution class of computational outcomes. For instance, in the "metalevel Bernoulli model" \citep{hay2014selecting}, a computational action involves sampling from a Bernoulli probability distribution. On the other hand, in the presented meta-BAMDP, a computation is graph expansion and we do not make any assumptions on how the subjective value increase per computation is distributed.
\end{enumerate}

\subsection{Dependence on the number of arms} \label{app:number_of_arms}
\begin{figure}
\centering
\subfloat[]{\includegraphics[width = 1.6in]{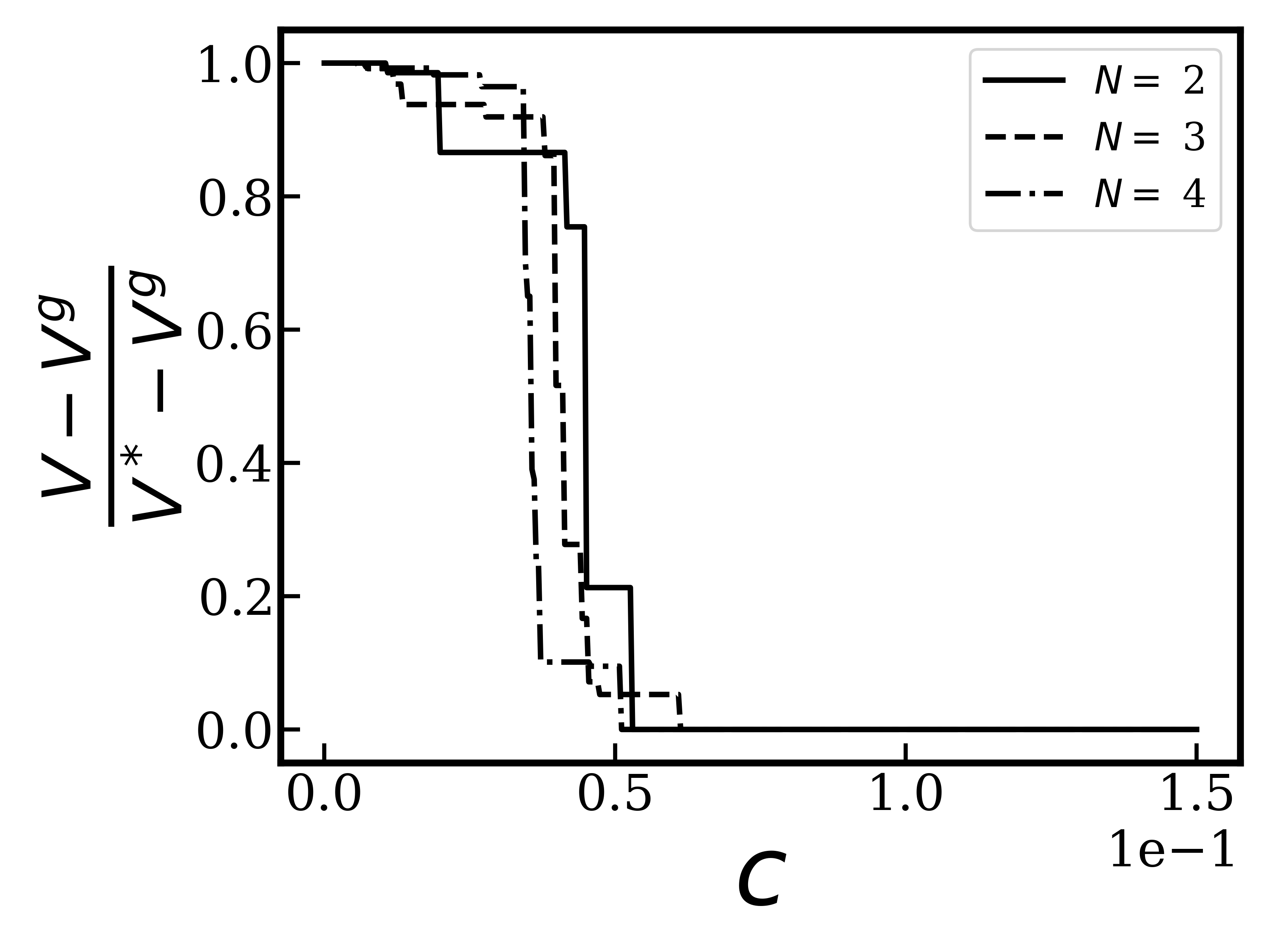}} 
\subfloat[]{\includegraphics[width = 1.6in]{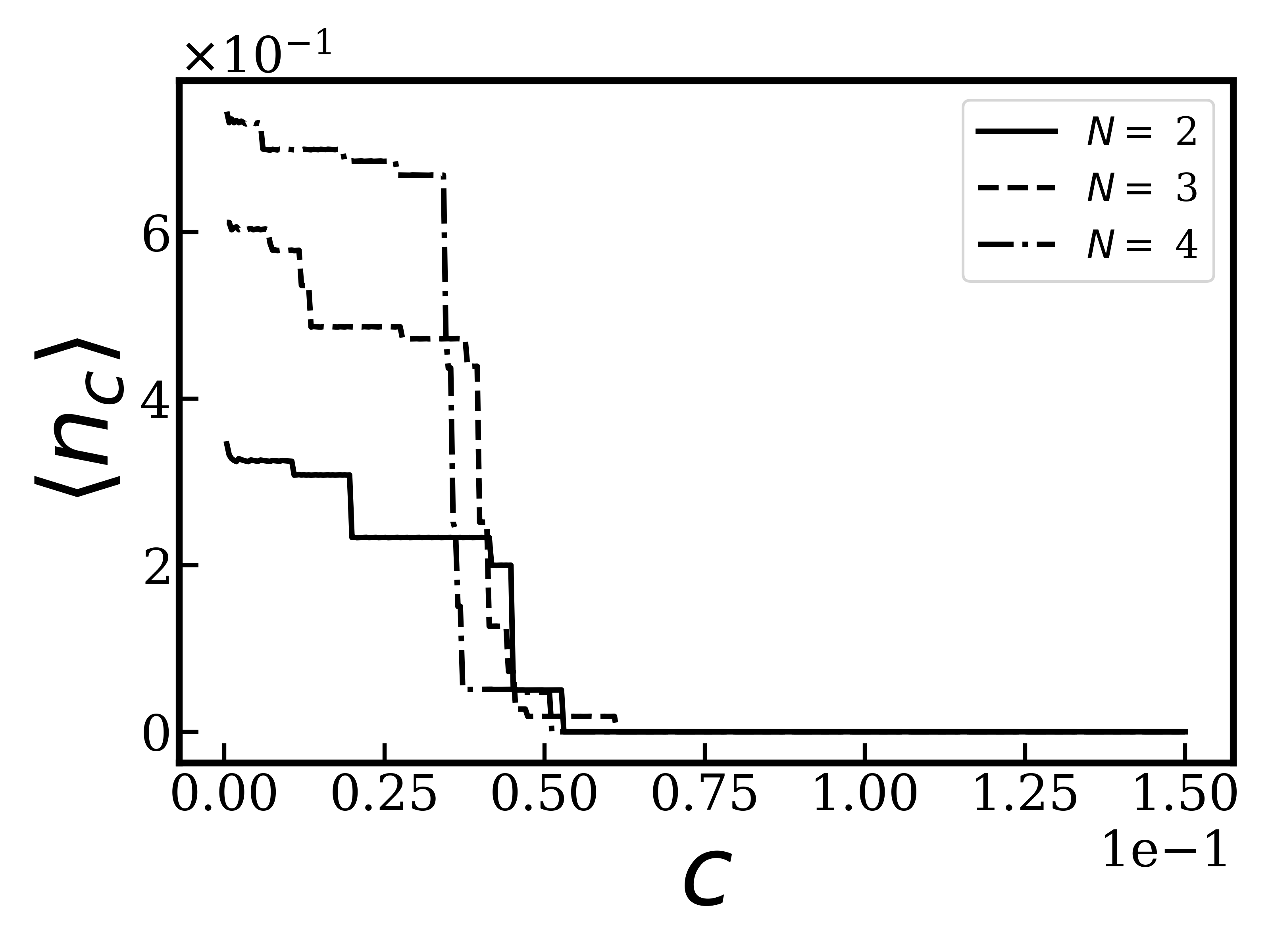}}
\subfloat[]{\includegraphics[width = 1.6in]{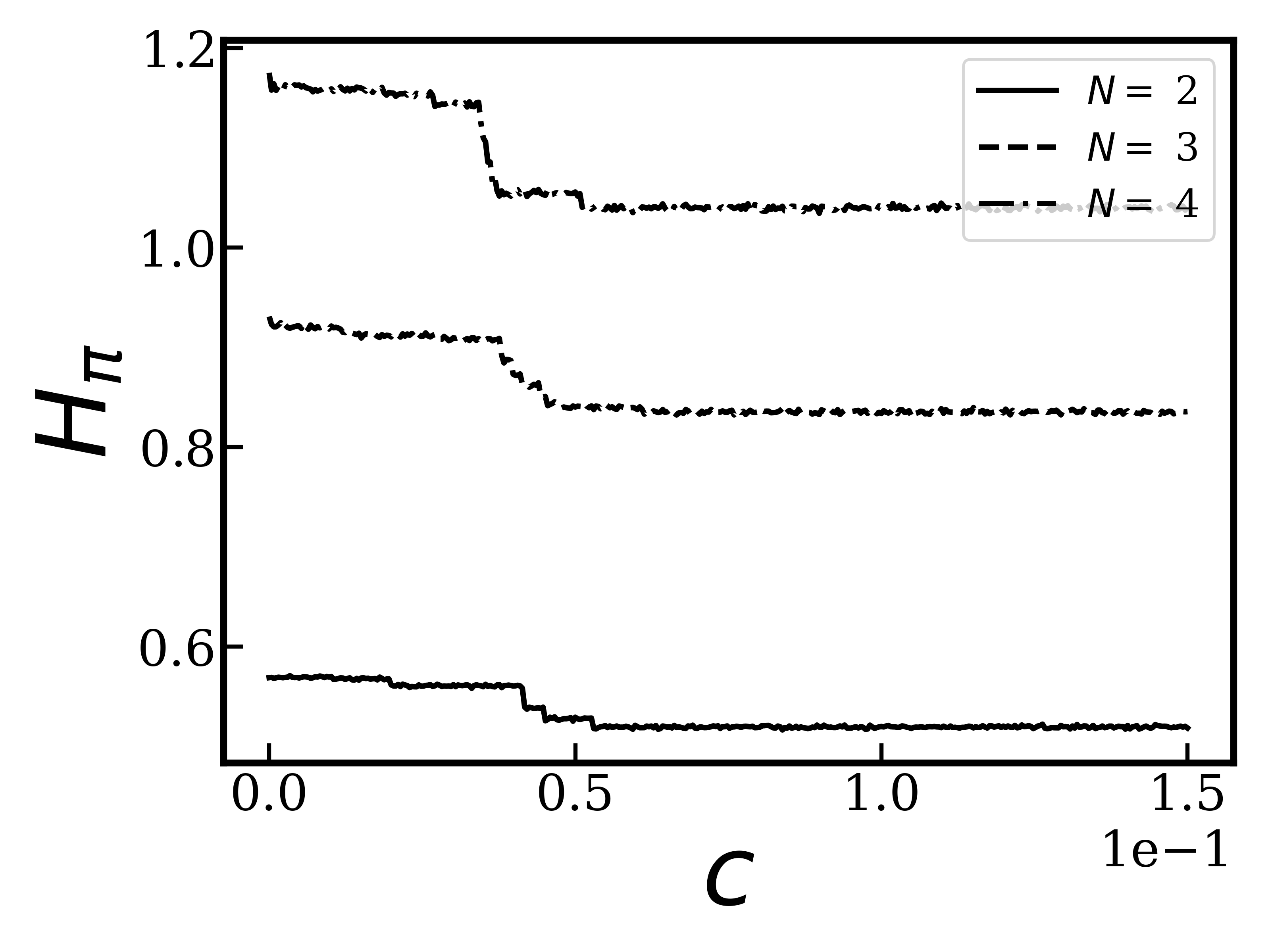}}
\subfloat[]{\includegraphics[width = 1.6in]{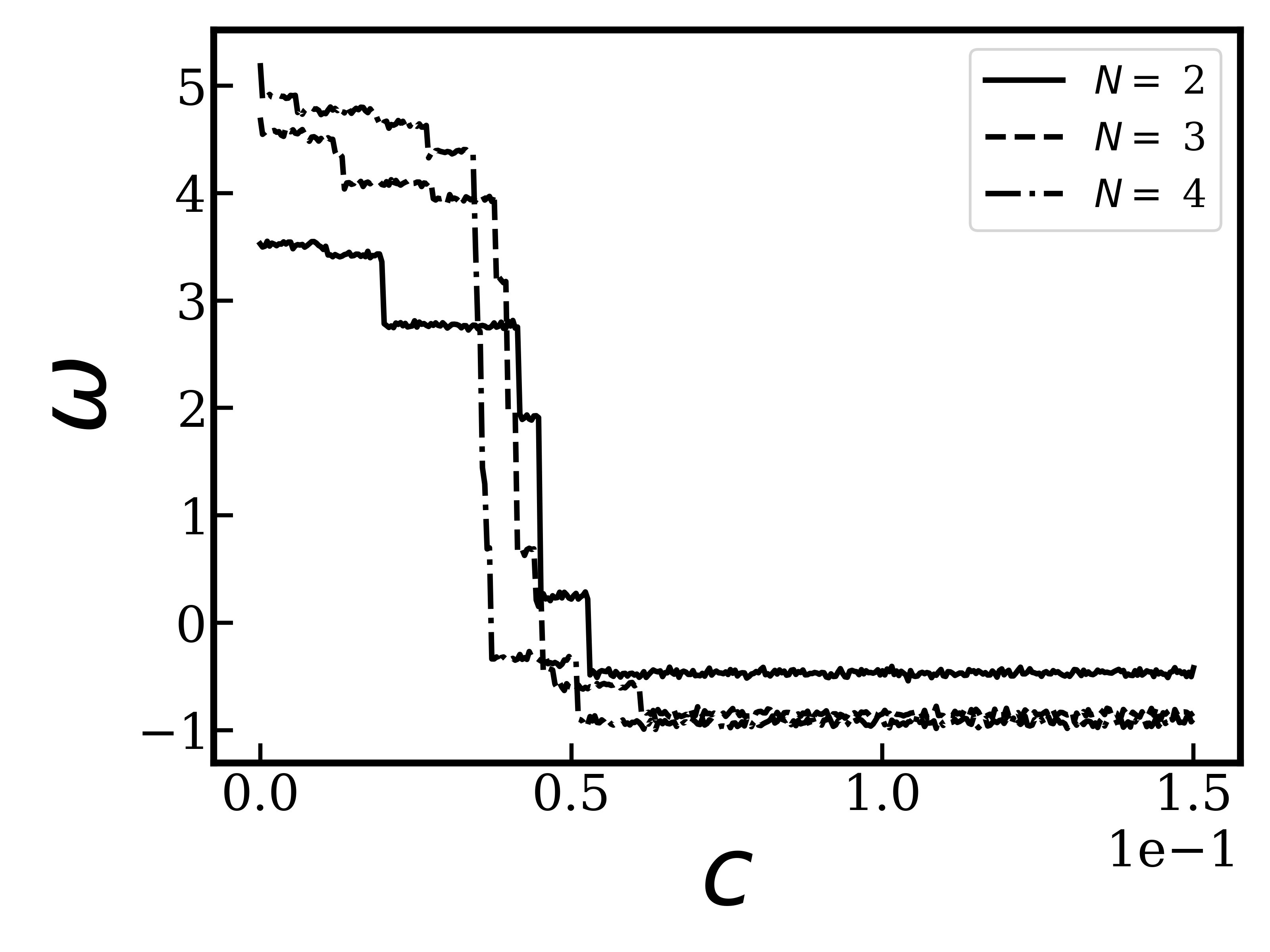}}
\caption{(a) Normalized reward gained as a function of computational cost for varying number of arms in tasks of length $T=9$. (b) Average number of computations performed as a function of computational costs for varying number of arms in tasks of length $T=9$. (c) Action entropy as a function computational costs with varying number of arms in tasks of length $T=9$ (averaged over $10^5$ simulation runs). (d) Best fit $\omega$ as a function of computational costs to behavior generated by meta-optimal policies in a symmetric environment with $p=0.5$ and $T=9$ (averaged over $10^5$ simulation runs).}
\label{fig:arms_vary}
\end{figure}
To explore the effects of varying number of arms on human behavior we first look at the normalized reward accrued (Fig. \ref{fig:arms_vary}(a)), which shows qualitatively the same behavior as for $N=2$ arms. However, we see that for higher number of arms the "drop" in the normalized reward happens sooner and is also sharper, suggesting that the agent is more sensitive to variations in computational costs with more number of arms. A similar trend can also be observed when we look at the average number of computations $\langle n_c\rangle$ (Fig. \ref{fig:arms_vary}(b)) as a function of computational costs. Additionally, we also note that agents compute more when the number of arms is high.

We also note that the qualitative features of human behavior adaptation to computational constraints - decreasing action entropy and directed exploration - are maintained in tasks with more number of arms. In Fig. \ref{fig:arms_vary}(c) we note that the action entropy reduces with computational costs, but increases with the number of arms. Additionally we note that the absolute drop in action entropy as a function of computational cost is higher in tasks with more arms. In Fig. \ref{fig:arms_vary}(d) we find that directed exploration decreases with increasing computational costs. Further we also note that for higher number of arms, the drop in directed exploratory behavior happens at smaller $c$ as compared to tasks with smaller number of arms. This is in line with experimental observations of human exploratory behavior adapting to environment size \citep{BROWN2022}.

\end{document}